 \documentclass[sigconf]{acmart}
\AtBeginDocument{%
  }

\usepackage{xcolor}
\usepackage{xspace}
\usepackage{amsmath,amsthm}
\usepackage{tikz,pgfplots}
\usepackage{mathtools}
\usepackage{verbatim}
\usepackage{multirow}
\usepackage{array}
\usepackage{algorithm}
\usepackage{algpseudocode}
\usepackage{wrapfig}

\algrenewcommand{\algorithmiccomment}[1]{{\footnotesize\color{black!40!white}\hfill$\triangleright$ #1}}

\newtheorem{problem}{Problem}

\newcommand{\tool}{\textsf{Clemont}\xspace}
\newcommand{\ior}{i.o.r.\xspace}
\newcommand{\Ior}{I.o.r.\xspace}
\newcommand{\IOR}{I.O.R.\xspace}

\newcommand{\RN}{\mathbb{R}}
\newcommand{\NN}{\mathbb{N}}

\newcommand{\bigO}{\mathcal{O}}

\setcopyright{acmlicensed}

\copyrightyear{2025}
\acmYear{2025}
\setcopyright{cc}
\setcctype{by}
\acmConference[KDD '25]{Proceedings of the 31st ACM SIGKDD Conference on Knowledge Discovery and Data Mining V.2}{August 3--7, 2025}{Toronto, ON,
Canada}
\acmBooktitle{Proceedings of the 31st ACM SIGKDD Conference on Knowledge Discovery and Data Mining V.2 (KDD '25), August 3--7, 2025, Toronto, ON, Canada}
\acmDOI{10.1145/3711896.3737054} \acmISBN{979-8-4007-1454-2/2025/08}

\begin{document}

\title{Monitoring Robustness and Individual Fairness}



 \author{Ashutosh Gupta}
 \affiliation{%
   \institution{IIT Bombay}
   \city{Mumbai}
   \country{India}}
 \email{akg@iitb.ac.in}
 
 \author{Thomas A.\ Henzinger}
 \affiliation{%
   \institution{ISTA}
   \city{Klosterneuburg}
   \country{Austria}}
 \email{tah@ist.ac.at}
 
 \author{Konstantin Kueffner}
 \affiliation{%
   \institution{ISTA}
   \city{Klosterneuburg}
   \country{Austria}}
 \email{konstantin.kueffner@ist.ac.at}
 
 \author{Kaushik Mallik}
 \affiliation{%
   \institution{IMDEA Software Institute}
   \city{Madrid}
   \country{Spain}}
 \email{kaushik.mallik@imdea.org}
 
 \author{David Pape}
 \affiliation{%
   \institution{Paris Lodron Universit\"at Salzburg}
   \city{Salzburg}
   \country{Austria}}
 \email{david.pape@stud.plus.ac.at}








\begin{abstract}
 In automated decision-making, it is desirable that outputs of decision-makers be robust to slight perturbations in their inputs, a property that may be called \emph{input-output robustness}.
    Input-output robustness appears in various different forms in the literature, such as robustness of AI models to adversarial or semantic perturbations and individual fairness of AI models that make decisions about humans.
    We propose runtime monitoring of input-output robustness of deployed, black-box AI models, where the goal is to design monitors that would observe one long execution sequence of the model, and would raise an alarm whenever it is detected that two similar inputs from the past led to dissimilar outputs. 
    This way, monitoring will complement existing offline ``robustification'' approaches to increase the trustworthiness of AI decision-makers.
    We show that the monitoring problem can be cast as the fixed-radius nearest neighbor (FRNN) search problem, which, despite being well-studied, lacks suitable online solutions.
    We present our tool \tool\footnote{\href{https://github.com/ariez-xyz/clemont}{https://github.com/ariez-xyz/clemont}}, which offers a number of lightweight monitors, some of which use upgraded online variants of existing FRNN algorithms, and one uses a novel algorithm based on binary decision diagrams---a data-structure commonly used in software and hardware verification.
    We have also developed an efficient parallelization technique that can substantially cut down the computation time of monitors for which the distance between input-output pairs is measured using the $L_\infty$ norm.
    Using standard benchmarks from the literature of  adversarial and semantic robustness and individual fairness, we perform a comparative study of different monitors in \tool, and demonstrate their effectiveness in correctly detecting robustness violations at runtime.
\end{abstract}



\keywords{Monitoring, individual fairness, adversarial robustness, semantic robustness, fixed-radius nearest neighbor search, trustworthy AI}
\maketitle

\section{Introduction}
AI decision-makers are being increasingly used for making critical decisions in a wide range of areas, including banking~\cite{rahman2023adoption,kaya2019artificial}, hiring~\cite{li2021algorithmic}, object recognition~\cite{serban2020adversarial}, and autonomous driving~\cite{chen2024end,yurtsever2020survey}. 
It is therefore crucial that they are reliable and trustworthy.
One of the general yardsticks of reliability is (global) \emph{input-output robustness}, which stipulates that similar inputs to the given AI model must lead to similar outputs.
This subsumes a number of widely used metrics, namely \emph{adversarial robustness} of image classifiers~\cite{mangal2019robustness}, requiring images that are pixel-wise similar be assigned similar labels, \emph{semantic robustness} of image classifiers~\cite{croce2020robustbench}, requiring images that capture similar semantic objects are assigned similar labels, and \emph{individual fairness} of human-centric decision-makers~\cite{dwork2012fairness,ruoss2020learning}, requiring individuals with similar features receive similar treatments.

Currently, input-output robustness of AI models is evaluated \emph{offline}, i.e., before seeing the actual inputs to be encountered during the deployment~\cite{dwork2012fairness,leino2021globally}, and it is required that the model be robust either with high probability with respect to a given input data distribution---the \emph{probabilistic} setting, or against all possible inputs---the \emph{worst-case} setting.
In practice, these offline robustness requirements of AI models are impossible to achieve due to various reasons.
For example, probabilistic robustness is problematic under data distribution shifts~\cite{taori2020measuring}, and worst-case robustness is tricky in classification tasks due to output transitions near class boundaries~\cite{kabaha2024verification}.

We propose a practical, \emph{runtime} variant of input-output robustness where robustness needs to be achieved on specific (finite) runs of deployed models, and a given \emph{run} violates robustness if two similar inputs from the past produced dissimilar outputs.
It is easy to see that (worst-case) offline input-output robustness implies runtime input-output robustness, but not the other way round: an AI model that is unrobust in the offline setting can still produce robust runs if the unrobust input pairs do not appear in practice.
Naturally, runtime input-output robustness is immune to data distribution shifts, and remains unaffected by class boundaries if inputs near the boundaries do not appear at runtime.
This way, runtime input-output robustness accounts for only those inputs that matter.

\begin{figure}
    \centering
    \begin{tikzpicture}
        \draw   (0,0)   rectangle   (2.5,1)   node[pos=0.5,align=center]  {AI\\ Decision-Maker};
        \draw   (-1.8,2)   rectangle   (3.3,4.5);
        
        \draw   (-1.6,2.1)  rectangle   (1.5,4) node[pos=0.5]    {\includegraphics[scale=0.15]{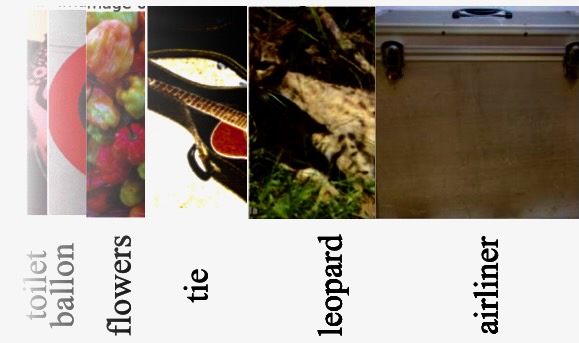}};
        \node   at  (-0.25,4.25)    {input-output history};
        
        \draw   (2,2.4)   rectangle   (3,3.9)  node[pos=0.5,align=center]   {FRNN\\ algo.};
        \draw[->]   (1.5,3)   --  (2,3);
        \draw[<-]   (1.5,3.3)   --  (2,3.3);
        \node   at  (2.1,4.7)   {\emph{Monitor}};

        \node   (in)   at  (-1.2,0.5) {\includegraphics[scale=0.1]{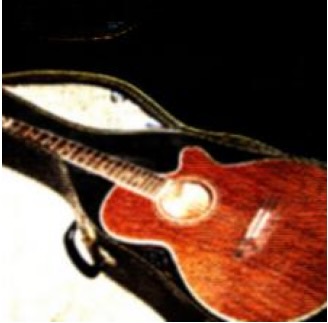}};
         \node  (in_txt)  at   (-1.2,1.3)     {new input};
        \draw[->]   (in)   --  (0,0.5);
        \draw[->]   (2.5,0.5)   --  (3.5,0.5);   \node  (out_txt)   at  (4,0.5)    {\textit{guitar}};

        \draw[->]   (in_txt.north)  --  (-1.2,1.8)  --  (2.3,1.8)   --  (2.3,2.4);
        \draw[->]   (out_txt.north) --  (4,1.6) --  (2.7,1.6)   --  (2.7,2.4);

        \draw[->]   (3,3.25)  --  (3.5,3.25);
        \node[align=center]   at  (5.1,3.25)  {\textbf{Violation!!!}\\ \\
                                                    \begin{tabular}{m{0.6cm} m{0.5cm} m{1.2cm}}
                                                        earlier: & \includegraphics[scale=0.06]{figures/illustrations/illustration_monitor_observation} & $\mapsto$ \texttt{tie}  \\
                                                        now: & \includegraphics[scale=0.06]{figures/illustrations/illustration_monitor_observation} & $\mapsto$ \textit{guitar}
                                                    \end{tabular}};
        
    \end{tikzpicture}
    \caption{Schematic diagram of input-output robustness monitors. The monitor stores the history of seen input-output pairs, and after arrival of each new input and the respective output of the AI decision-maker, uses a fixed-radius nearest neighbor (FRNN) search algorithm to check if any ``close'' input from the past gave rise to a ``distant'' output. The predictions in the figure are from the AlexNet~\cite{krizhevsky2012imagenet} model.}
    \label{fig:schematic diagram:monitor}
\end{figure}

We propose \emph{monitoring} of runtime input-output robustness.
The objective is to design algorithms---or monitors---which would observe one long input-output sequence of a given black-box decision-maker, and, after each new observation, would raise an alarm if runtime robustness has been violated by the current run.
In addition, after detecting a violation, the monitor would present a \emph{witness set}, which is the set of every past input that is similar to the current input but produced a dissimilar output.
The witness set can then be scrutinized by human experts, and measures can be taken if needed.

Monitoring has been extensively used to improve trustworthiness of software systems in other areas of computer science, including safety assurance in embedded systems~\cite{bartocci2018lectures} and bias mitigation in human-centric AI~\cite{henzinger2023monitoring,henzinger2023runtime,albarghouthi2019fairness}.
Like in these applications, monitoring is meant to \emph{complement}---and not replace---the existing offline measures of input-output robustness.
In fact, our experiments empirically show that offline robustness improves runtime robustness by a significant margin.
Yet, AI models that were designed using state-of-the-art offline robust algorithms still showed considerable runtime robustness violations.
Without monitoring, these violations would go undetected.
Some offline algorithms \emph{verify} the absence of robustness violations of trained models using formal methods inspired approaches~\cite{khedr2023certifair,meng2022adversarial}, which usually do not scale for large and complex models. 
As monitors treat the monitored systems as black-boxes, their performances remain unaffected by model complexities, making them essential tools when no verification approach would scale.
In fact, our monitors are shown to scale for examples up to more than 100,000 feature dimensions and for neural networks with more than 350 million parameter.
Such systems are beyond the reach of static verification approaches.

We show that the algorithmic problem of monitoring input-output robustness boils down to solving the well-known fixed-radius nearest neighbor (FRNN) search problem at each step of seeing a new input-output pair, as illustrated in Figure~\ref{fig:schematic diagram:monitor}.
In FRNN, we are given a point $p$, a set of points $S$, and a constant $\epsilon>0$, where $p$ and $S$ belong to the same metric space, and the objective is to compute the set $S'\subseteq S$ which is the set of all points that are at most $\epsilon$-far from $p$.
For our monitoring problem, $S$ and $p$ are respectively the past and current input-output pairs at any given point, and the underlying metric space is designed in a way that two points are close to each other if they correspond to the violation of robustness.

Even though FRNN has been studied extensively, most existing algorithms consider the \emph{static} setting where $S$ remains fixed.
Usually, the static FRNN algorithms from the literature are concerned with building fast \emph{indexing schemes} for $S$, such that the process of computing $S'$ is efficient.
For monitoring, we need the \emph{dynamic} variant, where $S$ is growing with incoming input-output pairs, and $p$ is the current input-output pair.
The na\"ive approach to go from the static to the dynamic setting would be to recompute the index at each step after the latest point is added to $S$, but this will cause a substantial computational overhead in practice.

Our contributions are as follows.
\begin{itemize}
    \item We present a practical solution for upgrading existing static FRNN algorithms to the dynamic setting through periodic recomputation of indices. 
    \item We present a new dynamic FRNN algorithm based on the symbolic data structure called binary decision diagram (BDD) used in hardware and software verification.
    \item We present a parallelized FRNN algorithm that substantially boosts the computational performances of our monitors.
    \item We implemented our monitors in the tool CLEMONT, and show that it can detect violations within fraction of a second to a few seconds (per decision) for real-world models with
more than 350M parameter and 150.5k input dimensions.
\end{itemize}

\section{Related Work}

The property of robustness has been studied across various domains in computer science, most prominently, automata theory and AI. 
Robustness in automata theory appears in the study of transducers~\cite{henzinger2014lipschitz}, I/O-systems~\cite{henzinger2014lipschitz}, and sequential circuits~\cite{doyen2010robustness}. The notion of robustness used are structurally similar to the ones used in AI and include $\epsilon$-robustness, $(\epsilon, \delta)$-robustness, or Lipschitz robustness \cite{casadio2022neural}.
In AI we are interested in the robustness of a single model. Here we differentiate between local or global robustness~\cite{mangal2019robustness}, which differ by the domain where the robustness requirements must hold. Depending on applications, the existing robustness definitions go by names like semantic robustness~\cite{croce2020robustbench}, adversarial robustness~\cite{mangal2019robustness}, or individual fairness~\cite{dwork2012fairness,ruoss2020learning}.

The two general approaches ensuring the robustness of machine learning models are training \cite{bai2021recent} and verification \cite{meng2022adversarial}.
Training robust models is done using techniques such as regularization, curriculum learning, or ensemble learning \cite{bai2021recent}.
Those techniques often fail to provide strong robustness guarantees and those that do, mostly do so in expectation \cite{lahoti2019ifair,ruoss2020learning,benussi2022individual}. 
Verifying models for robustness is done using tool such as SMT solvers \cite{katz2017reluplex,katz2019marabou}, abstract interpretation \cite{gehr2018ai2}, mixed-integer programming \cite{tjeng2017evaluating}, or branch-and-bound \cite{wang2021beta}. A verified model is guaranteed to satisfy either local \cite{benussi2022individual} and global robustness \cite{john2020verifying,biswas2023fairify,yeom2020individual,urban2020perfectly}.
Those strong guarantees come at the cost of high computation time. In particular, they do not scale as the complexity of the classifier increases, and usually fail for networks with more than 1000 neurons \cite{meng2022adversarial,benussi2022individual,biswas2023fairify,elboher2020abstraction}.

Monitoring is a well-established topic in runtime verification of hardware, software, and cyber-physical systems~\cite{bartocci2018lectures}. 
Recently, monitoring has been extended to verify \emph{group} fairness properties of deployed AI decision-makers~\cite{henzinger2023runtime,henzinger2023monitoring,henzinger2023pmonitoring,albarghouthi2019fairness}, though monitoring individual fairness (an instance of \ior ) properties has appeared rarely~\cite{albarghouthi2019fairness}.
The difference between monitoring group fairness and individual fairness is in the past information that needs to be kept stored: while individual fairness (and \ior by extension) needs all decisions from the past, for group fairness, the explicit past decisions can be discarded and only some small statistics about them needs to be kept~\cite{henzinger2023runtime,henzinger2023monitoring,henzinger2023pmonitoring,albarghouthi2019fairness}.
The only known work on monitoring individual fairness proposes a simple solution similar to our brute-force monitor~\cite{albarghouthi2019fairness}, which we show to not suffice in many benchmarks.

We will see that some of our monitors build upon existing FRNN search algorithms, a survey of which is deferred to Section~\ref{sec:survey}.

\section{The Monitoring Problem}

\subsection{Input-Output Robustness (\IOR)}
We consider input-output robustness of AI classifiers, though our formulation can be easily extended for regression models.
AI \emph{classifiers} are modeled as functions of the form $D\colon X\to Z$, where $X$ is the \emph{input} space and $Z$  is the \emph{output} space, with the respective distance metrics $d_X$ and $d_Z$.
Each  $(x,D(x))$ pair will be referred to as a \emph{decision} of $D$.
\emph{Input-output robustness}, or \ior in short, of $D$ requires that a small difference in inputs must not result in a large difference in outputs.
\begin{definition}\label{def:i/o robustness}
    Let $D$ be a classifier.
    For given constants $\epsilon_X,\delta_Z > 0$ and two inputs $x, x' \in X$, the classifier $D$ is $(\epsilon_X, \delta_Z)$-\emph{input-output robust}, or $(\epsilon_X, \delta_Z)$-\ior\footnote{The acronym ``\ior'' will represent both the noun ``input-output robustness'' and the adjective ``input-output robust.''} in short, for $x$ and $x'$ if:
\begin{align}\label{eq:i/o robustness}
  d_X(x,x')\leq \epsilon_X \implies d_Z(D(x), D(x')) \leq \delta_Z.
\end{align}
\end{definition}
We drop the constants ``$\epsilon_X$'' and ``$\delta_Z$'' if irrelevant or unambiguous.
Usually, \ior is not defined with a pair of fixed inputs like in Definition~\ref{def:i/o robustness}, but rather as local or global requirements on the system, and these local and global variants can be retrieved from our definition of \ior as follows.
The classifier $D$ satisfies \emph{local} \ior with respect to a given input $x$, if for every $x'\in X$, $D$ is \ior for $x$ and $x'$, and $D$ satisfies \emph{global} \ior, if for every pair of inputs $x,x'\in X$, $D$ is \ior for $x,x'$ \cite{leino2021globally,casadio2022neural,john2020verifying}.
\Ior appears in different forms in the literature, which are reviewed below.




\begin{description}
    \item[Adversarial robustness.] The definition of adversarial robustness~\cite{croce2020robustbench} exactly mirrors \ior in \eqref{eq:i/o robustness} with $X$ usually being real-coordinate spaces with either $L_2$ or $L_\infty$ norm.
    %
    \item[Semantic robustness.] 
    A decision-maker is \emph{semantically robust} if a small semantic change in its input does not significantly change the output~\cite{croce2020robustbench}, where two input images or input texts are semantically close if their semantic meanings are similar, although the distance between their feature values can be large.
    For measuring semantic robustness of the classifier $D\colon X\to Z$, we use a separate AI model $S$ that maps every input $x\in X$ to a point $y$ in an intermediate lower-dimensional semantic embedding space $Y$.
    Two inputs $x,x'\in X$ are then semantically close if $S(x)$ and $S(x')$ are close to each other according to a given distance metric.
    Usually, $X$ and $Y$ are real-coordinate spaces with either $L_2$ or $L_\infty$ norms, and therefore semantic robustness reduces to \ior by defining $d_X(x,x') = \|S(x)-S(x')\|$ with the respective norm $\|\cdot \|$, and $\epsilon_X$ is assumed to be  specified.
    %
    \item[Individual fairness.]
    Individual fairness is a global robustness property defined to assess the fairness of classifiers making decisions about humans. Among many alternate definitions \cite{john2020verifying,dwork2012fairness,ruoss2020learning,khedr2023certifair}, we use the one of Biswas et.al.~\cite{biswas2023fairify}. 
\end{description}

\subsection{The New \emph{Runtime} Variant}


The existing local and global variants of \ior are \emph{offline} properties of classifiers, meaning they are evaluated before observing the actual inputs seen at runtime.
In practice, a classifier that is \emph{not} locally or globally \ior may still be acceptable, as long as the pairs of inputs that witness the unrobust behaviors do not appear at runtime.
This motivates us to introduce the third, \emph{runtime} variant of \ior, which is a property of a given decision sequence, and \emph{not} a property of the underlying classifier.
Here, a decision sequence of the classifier $D\colon X\to Z$ is any finite input-output sequence $(x_1,z_1),\ldots,(x_n,z_n)\in (X\times Z)^n$, for any $n>0$, such that for every $i\in [1;n]$, $D(x_i)=z_i$.
\begin{definition}\label{def:runtime i/o robustness}
    Let $D$ be a classifier and let $\epsilon_X,\delta_Z >0$. A decision sequence $(x_1,z_1)\ldots (x_n,z_n)$ of $D$ is \emph{runtime $(\epsilon_X, \delta_Z)$-\ior} if 
    \begin{align}\label{eq:runtime robustness}
   \forall i,j \in [1;n] \;.\; d_X(x_i,x_j)\leq \epsilon_X \implies  d_Z(D(x_i), D(x_j)) \leq \delta_Z.
\end{align}
\end{definition}
It is straightforward to show that runtime \ior is \emph{weaker} than global \ior:

\begin{theorem}\label{thm:runtime input-output robustness}
    Suppose $\epsilon_X, \delta_Z>0$ are constants and $X$ is infinite. 
    \begin{enumerate}
        \item Every decision-sequence of every globally $(\epsilon_X, \delta_Z)$-\ior classifier is runtime $(\epsilon_X, \delta_Z)$-\ior.
        \item If the decision-sequence of a classifier is runtime $(\epsilon_X, \delta_Z)$-\ior, the classifier is not necessarily globally $(\epsilon_X, \delta_Z)$-\ior.
    \end{enumerate}
\end{theorem}
The proof is in Appendix~\ref{sec:proof}.
Claim~(1) of Theorem~\ref{thm:runtime input-output robustness} implies that if a given decision sequence of a classifier is \emph{not} runtime \ior, then the classifier is surely \emph{not} globally \ior
On the other hand, Claim~(2) suggests that if the decision sequence \emph{is} runtime \ior, we will not be able to conclude whether the classifier is globally \ior or not.
Furthermore, from Definition~\ref{def:runtime i/o robustness}, as soon as a decision sequence violates runtime \ior, so will every future extension of the sequence, regardless of the decisions that will be made in future.





\subsection{Monitoring Runtime \IOR}
We consider the problem of \emph{online} monitoring of runtime \ior of classifiers.
The goal is to design a function---the \emph{monitor}---that observes one long decision sequence of a black-box classifier, and after observing each new decision $(x,z)$,
outputs the set of every past decision $(x',z')$ such that $x$ and $x'$ are close but $z$ and $z'$ are not.
If the monitor always outputs the empty set while observing a given long decision sequence, then the sequence is runtime \ior


\begin{problem}[Monitoring runtime \ior]\label{prob:monitoring i/o robustness}
    Let $D\colon X\to Z$ be an arbitrary (black-box) classifier and let $\epsilon_X, \delta_Z>0$ be constants.
    Compute the function $M\colon (X\times Z)^+\times (X\times Z) \to 2^{(X\times Z)}$ such that for every finite sequence of past decisions $\rho = (x_1,z_1)\ldots (x_n,z_n)$ of $D$, and for every new decision $(x_{n+1},z_{n+1})$,
    \begin{multline*}
        M(\rho,(x_{n+1},z_{n+1})) = \big\lbrace (x_i,z_i), i\in [1;n] \mid \\
        d_X(x_{n+1},x_i)\leq \epsilon_X\land d_Z(z_{n+1}, z_i) > \delta_Z \big\rbrace.
    \end{multline*}
    The function $M$ will be called the \emph{\ior monitor}.
\end{problem}
Monitoring runtime \ior offers an added level of trustworthiness in AI decision making, especially when the underlying decision maker is not known to be globally \ior
One possibility is that the outputs of the monitor can be sent for scrutiny by human experts, so that necessary steps can be taken.
Without monitoring, robustness violations would go undetected, and could manifest in greater risks and loss of trustworthiness.

\subsection{Reduction to Fixed-Radius Nearest Neighbor}
Problem~\ref{prob:monitoring i/o robustness} reduces to the online \emph{fixed-radius nearest neighbor} (FRNN) problem stated below:
 
\begin{problem}[FRNN monitoring]
\label{prob:fixed-radius nearest neighbor}
    Let $Q$ be a set equipped with the distance metric $d_Q$ and $\epsilon_Q>0$ be a given constant.
    Compute the function $M\colon Q^+\times Q\to 2^Q$ such that for every sequence of past points $\rho = q_1\ldots q_n\in Q^+$, and for every new point $q_{n+1}\in Q$,
    \begin{align*}
        M(\rho,q) = \{q_i, i\in [1;n]\mid d_Q(q_{n+1},q_i)\leq \epsilon_Q \}.
    \end{align*}
    The function $M$ will be called the \emph{FRNN monitor}.
\end{problem}
Problem~\ref{prob:monitoring i/o robustness} reduces to Problem~\ref{prob:fixed-radius nearest neighbor} by using $Q = X\times Z$, $\epsilon_Q=\epsilon_X$, and by defining the metric $d_Q$ as follows:
For every $(x,z),(x',z')\in Q$,
\begin{align*}
    d_Q((x,z),(x',z'))\coloneqq 
        \begin{cases}
            d_X(x,x')   &   \text{if }  d_Z(z_{n+1}, z_i) \geq \delta_Z\\
            \infty      &   \text{otherwise.}
        \end{cases}
\end{align*}
The advantage of stating the monitoring problem using Problem~\ref{prob:fixed-radius nearest neighbor} instead of using Problem~\ref{prob:monitoring i/o robustness} is simplicity, and from now on, the ``monitoring problem'' will refer to Problem~\ref{prob:fixed-radius nearest neighbor} unless stated otherwise.

\section{Preliminaries of FRNN Algorithms}
\label{sec:survey}
We review the existing FRNN algorithms with a focus on the ones used by our monitors.
We use the notation from Problem~\ref{prob:fixed-radius nearest neighbor}, where $Q$ is a set with the distance metric $d_Q$, and $\epsilon_Q>0$ is given.

\subsection{Brute-Force (BF) FRNN}
The most straightforward solution of Problem~\ref{prob:fixed-radius nearest neighbor} is the brute-force algorithm, which simply stores the set of seen points in memory, and after seeing a new point from $Q$, performs a brute-force search to collect all $\epsilon_Q$-close points. 
If $Q$ is a $d$-dimensional real-coordinate space, then clearly the time complexity at the $n$-th step is $\mathcal{O}(d\cdot n)$, since the new point must be compared with $n$ other $d$-dimensional points. 
In Section~\ref{sec:static FRNN algorithms}, we will review some FRNN algorithms with asymptotically better complexities but significantly higher overhead costs.
Because of this, for small $n$, an efficient implementation of the brute-force approach is capable of outperforming other alternatives. 
This will be visible in our experiments as well where we use the highly optimized brute-force similarity search algorithm implemented in Meta's Faiss library~\cite{douze2024faiss}.


\subsection{\emph{Static} FRNN with Indexing}
\label{sec:static FRNN algorithms}
Although FRNN is a well-studied problem, most existing \emph{non}-brute-force approaches consider the static, one-shot version of Problem~\ref{prob:fixed-radius nearest neighbor}, namely the setting where the nearest neighbors will be searched once, and the set of seen points is not accumulating.
In Section~\ref{sec:periodic indexing}, we will present FRNN monitors that use static FRNN algorithm as their back-ends, and, in principle, \emph{any} off-the-shelf static FRNN algorithm can be used.
For concrete experimentation and as a proof-of-concept, we chose two representative static algorithms, namely the classic $k$-dimensional tree or \emph{$k$-d tree algorithm} and the \emph{sorting-based nearest neighbor} algorithm.
These algorithms improve over the brute-force alternative by storing the given points in efficient data structures, aka \emph{indexes}, such that searching for nearest neighbors becomes efficient.
We will assume that in Problem~\ref{prob:fixed-radius nearest neighbor}, $Q=\RN^d$  for some dimension $d\in \NN$, and $d_Q$ is either the $L_{2}$-norm or the $L_{\infty}$-norm.
Although there are FRNN algorithms for general metric spaces~\cite{clarkson1997nearest,yianilos1993data}, they will be redundant for most use cases of \ior. 
Let $\{q_1,\ldots,q_n\}\eqqcolon\mathcal{D}\subseteq \mathbb{R}^d$ be the given set of past points, and $q_{n+1}$ be the new point as described in Problem~\ref{prob:fixed-radius nearest neighbor}.

\smallskip
\noindent\textbf{$k$-d trees.}
$k$-d trees are binary trees for storing the given set of points $\mathcal{D}$ in a $k$-dimensional space (for us $k=d$ and the space is $Q$).
Each leaf node of a $k$-d tree contains a set of points in $\mathcal{D}$.
Each internal node divides the space $Q$ into two halves using a hyperplane, and points in $\mathcal{D}$ that are on the left side of the hyperplane are stored in the left sub-tree, whereas the points that are on the right side are stored in the right sub-tree.
The construction of a $k$-d tree from $\mathcal{D}$ takes $\bigO(dn\log n)$ time. 
$k$-d trees are not optimized for finding all neighbors within a given radius, but rather for identifying a fixed number nearest neighbors, which takes $\bigO(\log n)$ time on an average and still $\bigO(n)$ time in the worst case.
If we modify $k$-d trees for the purpose of FRNN queries, then each query would take $\bigO(d \cdot n^{1-1/d} + d\cdot m)$ time, where $m$ is the number of neighbors which are $\epsilon_Q$-close to the given input point.
If $m$ is large and approaches $n$, then the time complexity approaches $\bigO(n)$---the same as the brute-force approach.
In practice, most data sets are sparse and $m$ is usually small.
Furthermore, $k$-d tree-based FRNN is superior to brute-force when $d$ is small but $n$ is large, whereas both become equivalent (modulo the additional indexing overhead of $k$-d trees) when $d$ is large and $n$ is small.


\smallskip
\noindent\textbf{Sorting-based nearest neighbor (SNN).}
The recently developed SNN algorithm uses a sorting-based indexing scheme that is faster to build than $k$-d trees.
In particular, the indexing step requires $\bigO(n\log n + nd^2)$-time. The algorithm computes an ascending sequence of key values, each value corresponding to a point in the dataset $\mathcal{D}$.
The property satisfied by the key values is, that if two points have key values $\epsilon$-far apart, then their $L_2$ distance must also be greater than $\epsilon$.
The FRNN queries exploit this relationship. First, the key value of the input point is computed, requiring $\mathcal{O}(d^2)$-time. 
Second the algorithm can safely discards all points with key values $\epsilon$-far from the key value of the input point. This can be done efficiently using binary search in $\bigO(\log n)$-time. Then a brute force search is performed on the remaining points. 
The soundness of SNN relies on the chosen norm, although they are not restricted to the $L_2$ norm, the $L_{\infty}$ norm is not supported \cite{chen2024fast}.

\subsection{Survey of Other FRNN Algorithms}

We chose $k$-d tree~\cite{bentley1975multidimensional} and SNN~\cite{chen2024fast} as representative indexing algorithms, where the indexing of $k$-d trees uses an implicit partitioning over the input space, while the indexing of SNN uses a partition-free approach.
Other algorithms using partitioning-based indexing include R trees \cite{dasgupta2013randomized,guttman1984r}, ball trees \cite{omohundro1989five}, cover trees \cite{beygelzimer2006cover}, general metric trees \cite{ciaccia1997m}, and GriSpy~\cite{chalela2021grispy}, and other algorithms using partition-free indexing include the work by Connor et.al.~\cite{connor2010fast}.
Any of these algorithms could be used in our monitor, and as a general rule of thumb, the partition-based approaches will face higher computational blow-up than the partition-free methods for monitoring systems with high-dimensional input spaces~\cite{chen2024fast}.

The key aspect in monitoring runtime \ior is scalability, which will require us to use optimized FRNN algorithms that are fast even for high-dimensional data, such that the nearest neighbor search consistently ends \emph{before} the arrival of the next input.
Multiple strategies could be used for improving scalability.

A first alternative would be to use off-the-shelf \emph{parallelized FRNN algorithms}, which include works on both CPU-based~\cite{kamel1992parallel,cao2020improved,choi2010parallel,men2024pkd,blelloch2022parallel,yesantharao2021parallel} and GPU-based~\cite{prasad2015gpu,you2013parallel,nagarajan2023rt,mandarapu2024arkade} parallelization.
These algorithms have their own strengths and weaknesses, and the choice of the appropriate algorithm will ultimately be driven by the application's requirements.
For instance, algorithms based on recent developments in GPU hardware tend to be limited to 3 dimensions \cite{mandarapu2024arkade}, while CPU-based algorithm are more flexible \cite{blelloch2022parallel}.

In Section~\ref{sec:periodic indexing}, we will build FRNN monitors by periodically recomputing indexes of static FRNN algorithms.
A faster alternative would be to use \emph{dynamic FRNN algorithms} permitting incremental updates to the indexing structure~\cite{fu2000dynamic,eghbali2019online,ciaccia1997,blelloch2022parallel,yesantharao2021parallel}.
Recent works present incremental indexing of Hamming weight trees for FRNN in Hamming spaces~\cite{eghbali2019online} and of $k$-d trees for FRNN in Euclidian metric spaces~\cite{blelloch2022parallel,yesantharao2021parallel}.
Most of these works provide approximate solutions (explained below), whereas we intend to find the exact set of nearest neighbors; the few existing algorithms~\cite{blelloch2022parallel,yesantharao2021parallel} for the exact setting will be incorporated in future editions of our monitors.


A third alternative would be to use \emph{approximate FRNN algorithms}, which would trade off monitoring accuracy with performance and may occasionally output false positives (reporting robustness violations even if there is none) or false negatives (not reporting robustness violations even if there are some).
The literature on approximate FRNN is vast, and includes approximation schemes that are either data-dependent~\cite{andoni2014beyond} or data-independent~\cite{andoni2008near}, and use various techniques ranging from input space dimensionality reduction~\cite{andoni2018approximate} to approximate tree-based space partitioning~\cite{lin2006random,beygelzimer2006cover,ram2019revisiting} to the hierarchical navigable small world search algorithm~\cite{malkov2018efficient}.
All these approaches can be integrated within our monitor, and may be useful if occasional false outputs are acceptable.

\section{FRNN Monitoring via Periodic Indexing}
\label{sec:periodic indexing}
Most existing, non-brute-force FRNN algorithms build index structures for storing the set of input points, and these indexes usually do not support incremental updates that would be suitable for monitoring input \emph{sequences}. 
Recomputing the entire index at each step would incur a substantial computational cost and is infeasible.
We resolve this in Algorithm~\ref{alg:periodic indexing} by using a simple practical approach, namely re-indexing the FRNN data structure only periodically, after the interval of a given fixed number of inputs $\tau>0$.
The algorithm stores the past inputs in two separate memories, namely a long-term memory $L$ and a short-term memory $S$.
The long-term memory $L$ is updated periodically and stores past inputs that appeared \emph{before} the last update of $L$, while the short-term memory $S$ is the ``buffer'' that stores inputs that appeared \emph{after} the last update of $L$.  
Every time a new input $q$ appears, we need to search for its neighbors in the set $L\cup S$.
We delegate the search over $S$ to the brute-force approach and the search over $L$ separately to a static FRNN approach.
Every time the size of $S$ reaches $\tau$, we transfer the points in $S$ to $L$, reset $S$ to the empty set, rebuild the index of the static FRNN algorithm using the updated $L$, and continue with the next input.

In our experiments, for the static part, we compared $k$-d tree-based FRNN and SNN as proof-of-concept, although \emph{any} FRNN algorithm could be used.
For a fixed dimension of the input space and for a large set of past input points, the performances of $k$-d trees and SNN are significantly better than the brute-force approach, making them suitable for searching over $L$.
The hyper-parameter $\tau$ needs to be selected in a way that it creates a balance between the increasing query time of the brute-force algorithm and the cost of re-indexing the static FRNN algorithm. 
The following method can be applied.
Let $n$ be a given number of input-output pairs, and suppose $f(n)$, $g(n)$, and $h(n)$ represent, respectively, the worst-case time complexities of computing the FRNN index, the search over the long-term memory, and the search over the short-term memory. Typically, $f(n) > h(n) > g(n)$, and there is the tradeoff between indexing too frequently (low $\tau$), paying the high price of $f(n)$ more often, and indexing too rarely (high $\tau$), paying the price of $h(n)$ more often while the lower amount $g(n)$ could be used. We can express the overall amortized complexity of the monitoring algorithm (per execution step) as $T(\tau) = (f(n+\tau) + \tau g(n) + \sum_{i=1}^\tau h(i))/\tau$. The functions $f$, $g$, and $h$ will depend on the particular algorithms being employed, but usually they are simple enough that we can find the optimal $\tau$ such that $T(\tau)$ is minimized.


\begin{algorithm}[t]
    \caption{FRNN Monitor: Static FRNN with Periodic Indexing}
    \label{alg:periodic indexing}
    \begin{algorithmic}[1]
        \Require Space $Q$, distance metric $d_Q$, constant $\epsilon_Q>0$
        \State $L\gets \emptyset$ \Comment{initialize long-term memory}
        \State $S\gets \emptyset$ \Comment{initialize short-term memory}
        \State $F \gets \mathit{StaticFRNN}(Q,d_Q,\epsilon_Q)$ \Comment{initialize static FRNN object}
        \State $F.\mathit{ComputeIndex}(L)$ \Comment{\emph{initial indexing}}
        \While{$\mathit{true}$} \Comment{monitoring begins}
            \State $q\gets \mathit{GetNewInput}()$ \Comment{$q$ is the new input}
            \State $U\gets F.\mathit{FRNN}(q)$ \Comment{static FRNN}
            \State $V\gets \mathit{BruteForceFRNN}(S,q;d_Q,\epsilon_Q)$ \Comment{brute-force FRNN}
            \State \textbf{output} $U\cup V$ \Comment{monitor's output}
            \State $S\gets S\cup \{q\}$ \Comment{update short-term memory}
            \If{$|S|=\tau$} \Comment{\emph{recompute index} for static FRNN?}
                \State $L\gets L\cup S$ \Comment{update long-term memory}
                \State $S\gets \emptyset$ \Comment{reset short-term memory}
                \State $F.\mathit{ComputeIndex}(L)$ \Comment{recompute index}
            \EndIf
        \EndWhile
    \end{algorithmic}
\end{algorithm}

\section{FRNN Monitoring using\\ Binary Decision Diagrams (BDD)}
\label{sec:bdd}
We now present a new FRNN monitoring algorithm that is suitable in the dynamic setting of monitoring, where the seen inputs are incrementally updated over time.
Our new algorithm gives rise to monitors that in many cases outperform the monitors using off-the-shelf FRNN algorithms with periodic indexing.
The ides is to perform a bi-level search, where the top-level search is \emph{fast but approximate}, and uses an indexing scheme with \emph{binary decision diagrams} (BDD), and the bottom-level search is \emph{slow but exact}, and uses the brute-force algorithm.
While the top-level search would quickly narrow down the search space, it can generate false positives, which would then be eliminated by the bottom-level brute-force search.
This algorithm works only for $L_\infty$ distance; extensions to other metrics is left for future work.



\smallskip
\noindent\textbf{BDDs for indexing.}
BDDs have been extensively used in hardware and software verification in computer science.
The syntactic description of BDDs is irrelevant and out of scope; see the work of Bryant~\cite{bryant2018binary} for reference.
We will use BDDs essentially as black-boxes to build an efficient data structure for the FRNN search.
To this end, a BDD is essentially a function of the form $f\colon \{0,1\}^n\to \{0,1\}$, where $n>0$ is the number of boolean input variables.
Let $Q$, $d_Q$, and $\epsilon_Q$ be as defined in Problem~\ref{prob:fixed-radius nearest neighbor}.

\begin{wrapfigure}{r}{0.3\linewidth}
    \centering
    \includegraphics[width=1\linewidth]{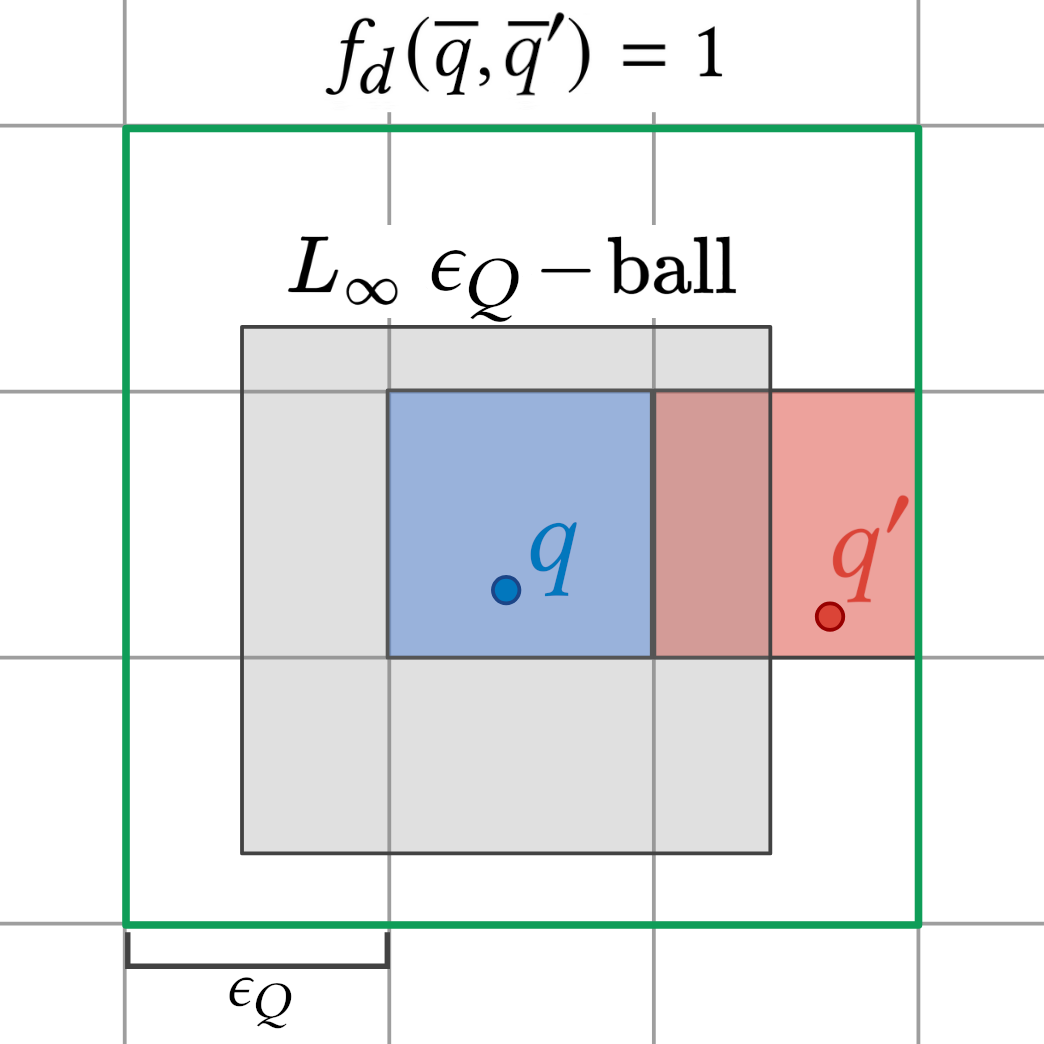}
    \caption{
    Discretization of $Q=\mathbb{R}^2$. The points have neighboring label, but are $\epsilon_Q$ apart.
    }
    \label{fig:falsepositives}
\end{wrapfigure}
Suppose $Q$ has some real and some categorical dimensions.
Assuming that the real dimensions have a bounded range, we will discretize the real dimensions into \emph{finitely} many (non-overlapping) intervals of width $\epsilon_Q$.
Each interval has a representative point as its label, and every other points inside this interval is ``approximated'' by the same label.
For example, in a data set on humans, ``age'' can be a real-valued feature, whose realistic range could be $0$--$120$.
If $\epsilon_Q=1$, the discrete intervals become $[0,1),[1,2),\ldots,[119,120)$.
For each interval, we can choose the lower bound as its label, i.e., the label of $[34,35)$ is $34$.
Given a real feature value $x$, we will use $\overline{x}$ to represent its respective label. 
The categorical dimensions (such as ``gender'') are assumed to have only finitely many values, and the ``label'' of each possible value is the value itself.
This way, we discretized $Q$ using a set of finitely many vectors of labels, call it $\overline{Q}$, where the $i$-th element of each vector is the respective label in the $i$-th feature dimension.
For a given input $q\in Q$, let $\overline{q}$ represent its corresponding label vector.

We now introduce the the BDD $f\colon \{0,1\}^n\to \{0,1\}$ which will store the set $W$ of past inputs using their discrete label vectors.
For this, we will require $n = \log(|\overline{Q}|)$ bits to encode the label vectors.
At each point, for a given $q\in Q$, and assuming $b_q$ is the boolean encoding of $\overline{q}$, $f(b)=1$ iff $q$ has been seen in the past.
After seeing every new input, $f$ can be updated using standard BDD operations~\cite{bryant2018binary}.
We also maintain a dictionary $\Delta$, that maps each label vector $v\in \overline{Q}$ to the set of seen points $q\in Q$ with label $v$.
We need another BDD for encoding the distance function $d_Q$, denoted as $f_d\colon \{0,1\}^{2n}\to \{0,1\}$, which takes as inputs two binary encoded label vectors $b,b'$ of two inputs $q,q'$, and outputs $f_d(b,b')=1$ iff each vector entry of $\overline{q}$ and $\overline{q'}$ are either all the same or are adjacent labels to each other; in our example of ``age,'' $34$ and $35$ are adjacent labels.
The BDD $f_d$ is constructed statically in the beginning using standard BDD operations~\cite{bryant2018binary}.

\smallskip
\noindent\textbf{The sketch of the algorithm.}
The FRNN monitor uses a \emph{hierarchical FRNN search}:
The BDDs sit at the top-level, and after each new input $q$ is observed, quickly checks if any past point had the same or adjacent labels.
There are three possible outcomes:
(a)~Neither $\overline{q}$ nor its neighbors appeared before, in which case the monitor outputs $M(\cdot,q)=\emptyset$, a case that we expect to experience most of the time.
(b)~The vector $\overline{q}$ appeared before but none of its neighbors did, in which case the monitor outputs $M(\cdot, q) = \Delta(\overline{q})$, since every two points with the same label vector are at most $\epsilon_Q$ apart.
(c)~Both (a) and (b) are false, i.e., some neighbor $\overline{q'}$ of $\overline{q}$ appeared before, which could mean either a true positive or false positive, since the distance between two points with neighboring labels may or may not be smaller than $\epsilon_Q$; see Figure~\ref{fig:falsepositives} for an illustration.

When Option~(c) is true and the result of the top-level search is inconclusive, the bottom-level brute-force search comes to rescue.
But now, the brute-force algorithm only needs to search within the set of seen inputs that have the neighboring labels of $\overline{q}$, which in most case will be significantly faster than the regular brute-force search over the entire set of seen inputs.
The pseudocode of the full algorithm is included in Appendix~\ref{app:FRNN using BDD}, which includes elementary BDD operations like disjunction and membership query.
Even though these operations have exponential complexity in the number of BDD variables $n$~\cite{bryant2018binary}, i.e., linear complexity with respect to $|\overline{Q}|$, still in practice, modern BDD libraries use a number of smart heuristics and have superior scalability.

\section{Performance Optimization: Parallelized FRNN}
Many existing FRNN algorithms, including $k$-d trees and our BDD-based algorithm, use partitioning of the input space, which causes a blow-up in the complexity with growing search space dimension~\cite{andoni2017nearest}. 
We present a parallelization scheme for performing nearest neighbor search over real-coordinate spaces equipped with the $L_\infty$ distance metric.
The idea is that for the $L_\infty$ metric, two points are $\epsilon$-close iff they are $\epsilon$-close in \emph{each} dimension.
This inspired us to decompose the given FRNN problem instance into multiple sub-instances of FRNN with fewer dimensions than the original problem.
Each sub-problem can be solved independently, and therefore in parallel, and it is made sure that when the solutions of the sub-problems are composed in a certain way, we obtain a solution for the original FRNN problem.
Our parallelization scheme works as a wrapper on any FRNN monitoring algorithm.
In our experiments, we demonstrate the efficacy of the parallelized version of both $k$-d tree-based FRNN monitors and BDD-based FRNN monitors.

\algblockdefx{ParStart}{ParEnd}{\textbf{Do in parallel}}{\textbf{End parallel}}
\begin{algorithm}
    \caption{Parallelized FRNN Monitoring}
    \label{alg:parallel}
    \begin{algorithmic}[1]
        \Require Space $Q$, distance metric $d_Q$, constant $\epsilon_Q>0$
        \State $S\gets \emptyset$ \Comment{initialize memory}
        \While{$\mathit{true}$} \Comment{monitoring begins}
            \State $q\gets \mathit{GetNewInput}()$ \Comment{$q$ is the new input}
            \State $\overline{S}\gets \mathit{AssignUniqueLabels}(S)$ \Comment{$\overline{S}\subset \mathbb{R}^{2n}\times\mathbb{N}$}\label{alg:step:label}
            \State $S_A \gets \{s\in \mathbb{R}^n\times \mathbb{N} \mid \exists q\in \overline{S}\;.\; (q.\mathrm{A},q.\mathrm{id}) = s\}$ \label{alg:step:project S to the first component} \Comment{projection on $A$}
            \State $S_B \gets \{s\in \mathbb{R}^n\times \mathbb{N} \mid \exists q\in \overline{S}\;.\; (q.\mathrm{B},q.\mathrm{id}) = s\}$ \label{alg:step:project S to the last component} \Comment{projection on $B$}
            \ParStart 
            \State $T_A\gets \text{FRNN}(S_A, (p.\mathrm{A},p.\mathrm{id});d_Q,\epsilon_Q)$ \Comment{local FRNN in $A$} \label{alg:step:solve FRNN on A}
            \State $T_B\gets \text{FRNN}(S_B, (p.\mathrm{B},p.\mathrm{id});d_Q,\epsilon_Q)$ \Comment{local FRNN in $B$} \label{alg:step:solve FRNN on B}
            \ParEnd
            \State $T \gets \{ t\in \mathbb{R}^{2n}\mid \exists i\in \mathbb{N}\;.\;\exists a\in T_A\;.\;\exists b\in T_B\;.\; (t.\mathrm{A},i) = a, (t.\mathrm{B},i)=b \}$ \Comment{composition of outputs of local FRNN monitors}\label{alg:step:compose}
            \State \textbf{output} $T$
            \State $S\gets S\cup \{q\}$ \Comment{update memory}
        \EndWhile
    \end{algorithmic}
\end{algorithm}

For simplicity, we explain the parallelized algorithm using two parallel decompositions of the given problem; the extension to arbitrarily many decompositions is straightforward.
Before describing the algorithm, we introduce some notation. 
We consider FRNN problems on the metric space $(\mathbb{R}^{2n},d_\infty)$ where $d_\infty$ is the $L_\infty$ norm.
For a given point $q = (r_1,\ldots,r_n,r_{n+1},\ldots,r_{2n})\in \mathbb{R}^{2n}$, we will write $q.\mathrm{A}$ and $q.\mathrm{B}$ to respectively denote the projections $(r_1,\ldots,r_n)$ and $(r_{n+1},\ldots,r_{2n})$.
We will use the augmented space $(\mathbb{R}^{2n}\times\mathbb{N},d_\infty')$, where the extra dimension $\mathbb{N}$ will be used to add unique labels to the points in $\mathbb{R}^{2n}$, and $d_\infty'$ equals to $d_\infty$ with the labels of the points ignored.
For a given point $q = (r_1,\ldots,r_n,r_{n+1},\ldots,r_{2n},m)$ in the augmented set, we will write $q.\mathrm{id}$ to denote the label $m$ of $q$, while $q.A$ and $q.B$ will as usual represent $(r_1,\ldots,r_n)$ and $(r_{n+1},\ldots,r_{2n})$, respectively.
We define the function $\mathit{AssignUniqueLabels}$ which takes as input a given \emph{finite} set $S\subset \mathbb{R}^{2n}$, and outputs a set $\overline{S}$ in the augmented space $\mathbb{R}^{2n}\times \mathbb{N}$ such that $\overline{S}$ exactly contains the elements of $S$ which are now assigned unique labels.


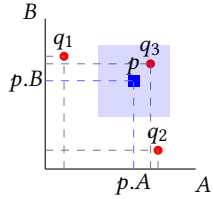
\begin{wrapfigure}{r}{0.3\linewidth}
    \centering
    \begin{tikzpicture}
        \draw[-]    (0,0)   --  (2,0)   node    at  (2.1,-0.2)  {$A$};
        \draw[-]    (0,0)   --  (0,2)   node    at  (-0.2,2.1)  {$B$};
        \fill[red] (0.25,1.5) circle (1.7pt);
        \node   at  (0.25,1.7)  {$q_1$};
        \draw[dashed,black!50!white]   (0,1.5) --  (0.25,1.5);
        \draw[dashed,black!50!white]   (0.25,0) --  (0.25,1.5);
        \fill[red] (1.5,0.25) circle (1.7pt);
        \node   at  (1.5,0.45)  {$q_2$};
        \draw[dashed,black!50!white]   (1.5,0) --  (1.5,0.25);
        \draw[dashed,black!50!white]   (0,0.25) --  (1.5,0.25);
        \fill[red] (1.4,1.4) circle (1.7pt);
        \node   at  (1.4,1.6)  {$q_3$};
        \draw[dashed,black!50!white]   (1.4,0) --  (1.4,1.4);
        \draw[dashed,black!50!white]   (0,1.4) --  (1.4,1.4);

        \fill[blue,opacity=0.15] (0.7,0.7)   rectangle   (1.65,1.65);
        \fill[blue] (1.1,1.1) rectangle   (1.25,1.25);
        \node   at  (1.175,1.375)  {$p$};
        \node   at  (1.175,-0.2)    {$p.A$};
        \node   at  (-0.25,1.175)    {$p.B$};
        \draw[dashed,blue!60!white] (1.175,0)   --  (1.175,1.175);
        \draw[dashed,blue!60!white] (0,1.175)   --  (1.175,1.175);
    \end{tikzpicture}
    \caption{Parallelized FRNN monitoring.}
    \label{fig:parallel}
\end{wrapfigure}
The parallelized algorithm is presented in Algorithm~\ref{alg:parallel}, and we explain it using a simple example in Figure~\ref{fig:parallel}.
Suppose $n=1$, i.e., the original FRNN problem is in the metric space $(\mathbb{R}^2,d_\infty)$.
Let the inputs be the set $S = \{q_1,q_2,q_3\}$ and the new point $p$, as shown in Figure~\ref{fig:parallel}, where the shaded region around $p$ represents its $\epsilon_Q$-neighborhood for a given $\epsilon_Q>0$ (as in Problem~\ref{prob:fixed-radius nearest neighbor}).
Clearly, the FRNN algorithm should output $T=\{q_3\}$ as the set of $\epsilon$-neighbors of $p$.

To solve this problem in parallel, we first assign unique labels (Line~\ref{alg:step:label} in Algorithm~\ref{alg:parallel}) to the points $q_1,q_2,q_3$; let $q_i.\mathrm{id}=i$ for every $i\in\{1,2,3\}$.
Then we project $S$ to the two individual dimensions $A$ and $B$, giving us the lower dimensional sets $S_A$ and $S_B$ (Lines~\ref{alg:step:project S to the first component} and \ref{alg:step:project S to the last component}), where we ensure that the labels of the points are preserved in the projections.
Now we solve---in parallel---the two single-dimensional FRNN instances $(S_A, (p.\mathrm{A},p.\mathrm{id}),\epsilon_Q)$ and $(S_B, (p.\mathrm{B},p.\mathrm{id}),\epsilon_Q)$, giving us the lower-dimensional sets of $\epsilon_Q$-close points $T_A$ and $T_B$, respectively (Lines~\ref{alg:step:solve FRNN on A} and \ref{alg:step:solve FRNN on B}).
Finally we compose $T_A$ and $T_B$ to obtain the final answer $T$ (Line~\ref{alg:step:compose}).
The key insight is that the $d_\infty$-norm suggests that two points $p$ and $q$ are $\epsilon_Q$-close iff they are $\epsilon_Q$-close in \emph{all} dimensions.
Therefore, if there is a point $q$ that is $\epsilon_Q$-close to $p $ in $A$ but not in $B$ or vice versa, then $q$ is not $\epsilon$-close to $p$ and hence is not included in $T$; this case applies to both points $q_1$ and $q_2$ in Figure~\ref{fig:parallel}.
As the point $q_3$ is $\epsilon$-close to $p$ in both dimensions, it is added to $T$.
Note that the additional identification labels of the points help us to perform this synchronized check across both $A$ and $B$ dimensions.


\section{Experimental Evaluation}

We implemented our algorithms in the tool \tool, and use it to monitor well-known benchmark models from the literature. 
On one hand, we demonstrate the monitors' effectiveness on real-world benchmarks (Section~\ref{subsec:effectiveness}), and on the other hand we demonstrate their feasibility in terms of computational resources (Section~\ref{subsec:performance}).

Different parts of the experiments were run on different machines. 
Our monitors for all our examples run on CPUs, and GPU-based implementations (especially the parallellized FRNN algorithm) are left for future works.
The only place GPUs were used are for training the models used in our experiments.
Our monitors were evaluated on personal laptops with 8GB memory for all examples other than the adversarial robustness example with ImageNet model, whose feature space is too large (150,000 features) for personal laptops, and we used machines with 256GB memory for this one experiment.



\subsection{Practical Applications of \IOR Monitoring}\label{subsec:effectiveness}
The experimental setup and the corresponding results are provided in Table~\ref{tab:app_tpr}.
For adversarial and semantic robustness, we picked a number of image data sets provided by RobustBench \cite{croce2020robustbench}, which is a standardized benchmark suite for comparing robustness of AI models.
For each of the data sets, we picked the best and the worst performing models, and provided them with a sequence of input images, some of which were deliberately modified with adversarial or semantic corruptions.
For individual fairness, we picked the standard fairness data sets, and used baseline and fair models from various existing works from the literature.
These models were then given a sequence of input features of individuals.
In all these experiments, we deployed our monitors to track the violation of the respective robustness or fairness conditions by the AI models.

\begin{figure}[H]
    \centering
    \begin{tabular}{cc}
        \includegraphics[width=0.575\linewidth]{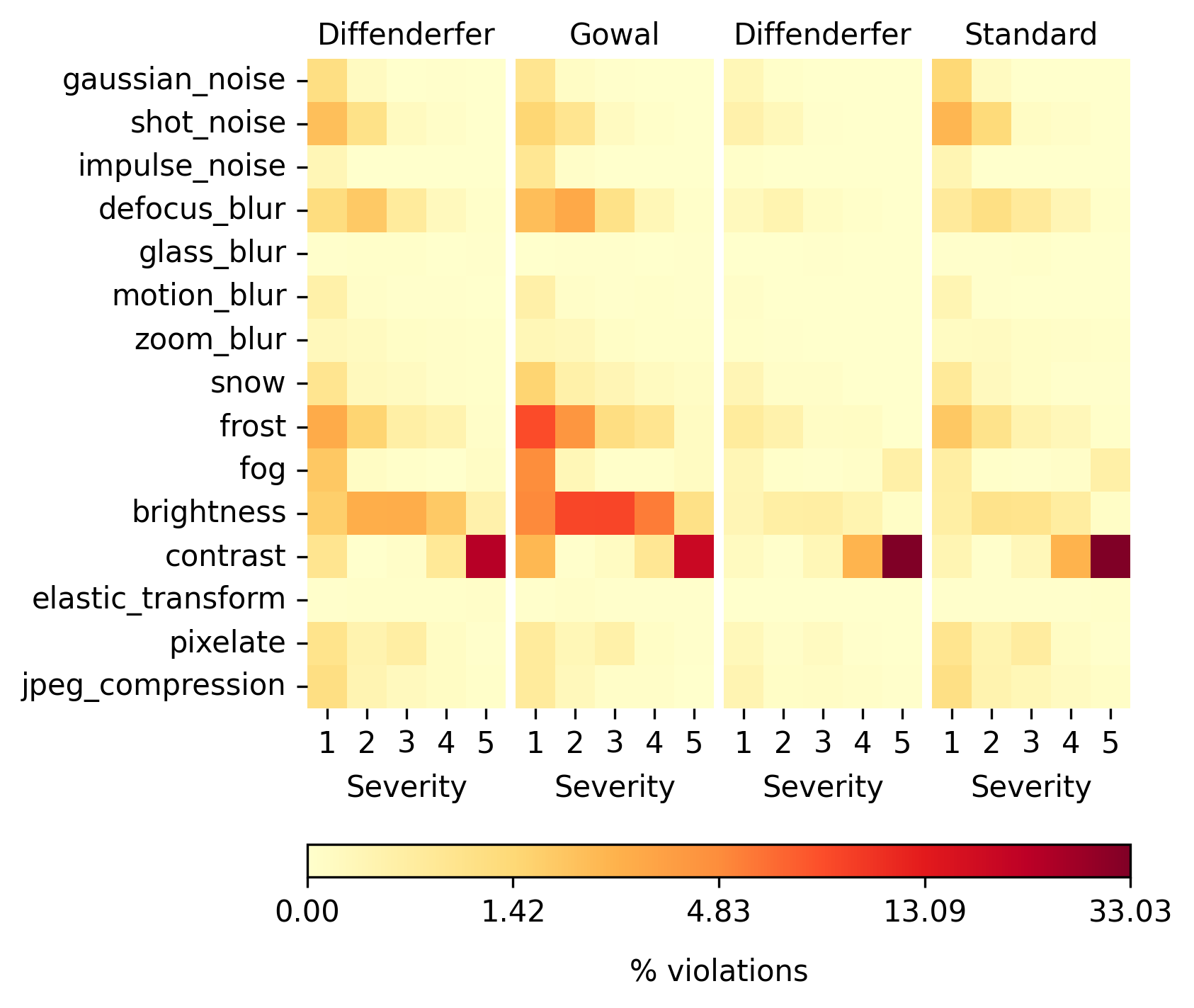} & 
        \includegraphics[width=0.33\linewidth]{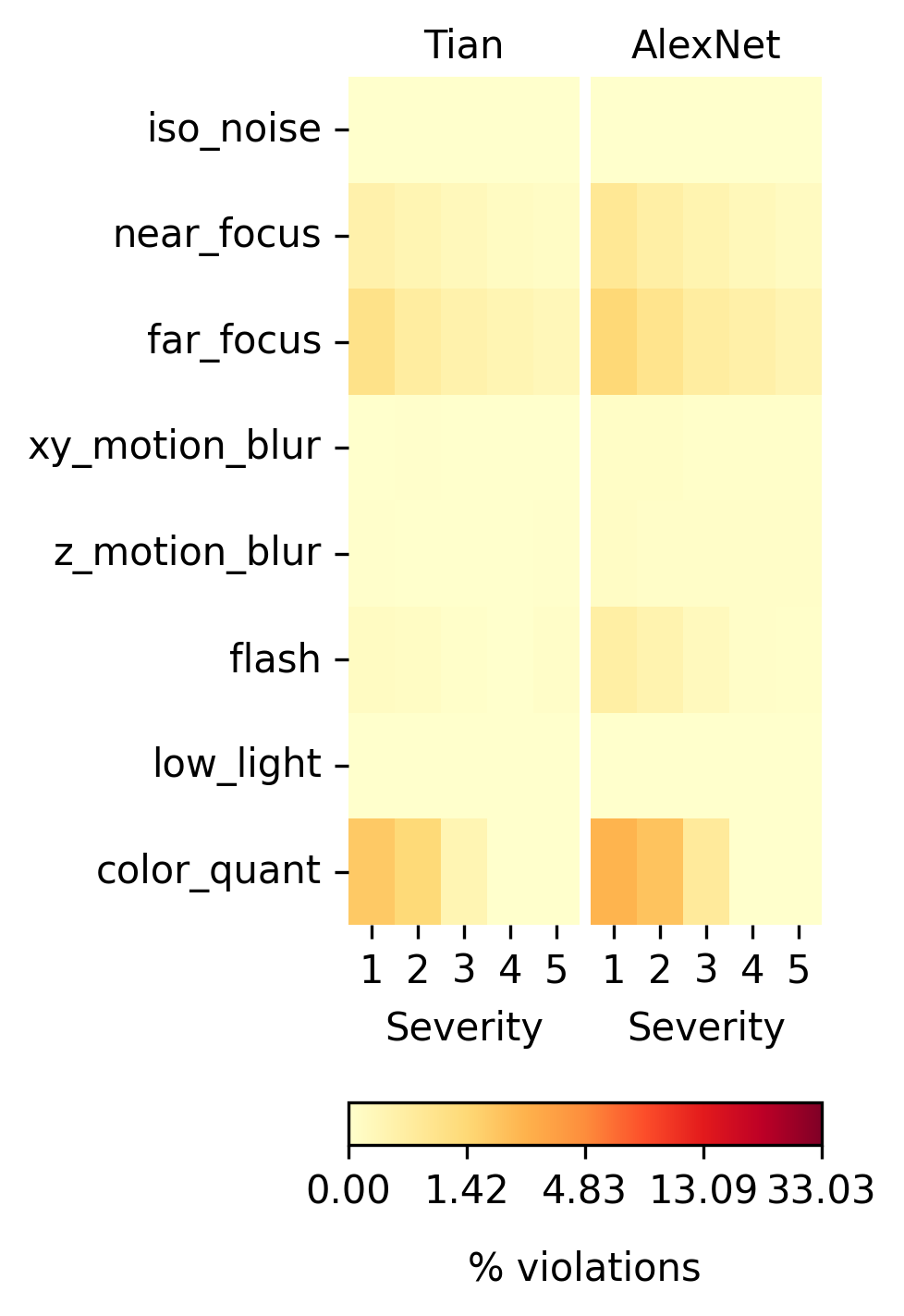} 
    \end{tabular}
    \caption{Semantic robustness violations identified during monitoring. Left to right: Robust \cite{diffenderfer2021winning} and base \cite{gowal2020uncovering} CIFAR100C model, robust \cite{diffenderfer2021winning} and base \cite{croce2020robustbench} CIFAR10C model, robust \cite{tian2022deeper} and base \cite{croce2020robustbench} ImageNet model.}
    \label{fig:heat}
\end{figure}

\smallskip
\noindent\textbf{Key Takeaways.}
The results are summarized in Table~\ref{tab:app_tpr}, and some concrete instances of detected semantic robustness violations are displayed in Appendix~\ref{sec:appendix:counter-examples}.
Our monitors always output correct answers by design, and therefore the only interesting quality metric is the violations rate, representing the average number of inputs for which the monitor detected \ior violations.
Except for a few cases of individual fairness, the violation rate is always positive, highlighting the need for monitoring as an additional safeguard.

The violation rate also confirms that robust or fair training algorithms do improve the runtime \ior in practice, which is expected.

For semantic robustness, we observe some interesting trends.
As the severity of the corruption increases, in most cases the violation rate goes down.
We suspect that this is because higher corruption increases the distance between the corrupted and the original image in the embedding space, and differences in output labels are not considered as robustness violation anymore.
However, for corruption of contrast, this trend does not hold.

We report the time and the memory requirements for monitoring the sequence of all the inputs from the entire benchmark data sets. We observe that the brute-force algorithm outperforms all other approaches in every category, which can be explained by the small number of data points for which brute-force excels (see Section~\ref{subsec:performance}).  

\begin{table*}
    \small
    \centering
    \setlength{\tabcolsep}{3pt}
    \begin{tabular}{c| c@{\hskip 1pt} c @{\hskip 1pt} c | c@{\hskip 1pt} c | c c c c c | c c c c c |  c@{\hskip 1pt} c@{\hskip 1pt} c@{\hskip 1pt} | c@{\hskip 1pt} c@{\hskip 1pt} c@{\hskip 1pt}}
        \toprule
         & Data Set  & $n$ & $d$  &  Norm  & $\epsilon$  & Base & \multicolumn{2}{c}{\# Param.} & \multicolumn{2}{c|}{Violations~(\%) } & Robust &  \multicolumn{2}{c}{\#Param.}  & \multicolumn{2}{c|}{Violations~(\%)}  &  \multicolumn{3}{c|}{Time per sample (ms)} & \multicolumn{3}{c}{Memory (MB)}\\
        \midrule
               &  &  &  & &  & &    & &  &  &   & && && BF  & $k$-d & $k$-d (16t) & BF  & $k$-d & $k$-d (16t) \\
        \midrule
        \multirow{3}{*}{\shortstack[1]{Adv. \\ Robust.}} &  CIFAR-10 &  \multirow{2}{*}{$20k$}  & \multirow{2}{*}{$3.1k$} &  \multirow{3}{*}{$L_{\infty}$} & \multirow{2}{*}{ $\frac{8}{255}$} 
        & \cite{croce2020robustbench} & \multicolumn{2}{c}{36M} & \multicolumn{2}{c|}{0.948} &\cite{bartoldson2024adversarial} & \multicolumn{2}{c}{366M} & \multicolumn{2}{c|}{0.196} & 3.71 & 6.09 &  1.87 &  2191 & 2852 &1946\\

        &  CIFAR-100 &   &  &   &  & \cite{rice2020overfitting} & \multicolumn{2}{c}{11M} & \multicolumn{2}{c|}{0.344} & 
        \cite{wang2023better} & \multicolumn{2}{c}{267M} & \multicolumn{2}{c|}{0.316} & 5.47  & 32.04 & 9.73 & 2057  & 2878 & 1945 \\

        & ImageNet &  $10k$  & $150.5k$ &   &  $\frac{4}{255}$ & \cite{croce2020robustbench} & \multicolumn{2}{c}{26M} & \multicolumn{2}{c|}{0.767}
        & \cite{amini2024meansparse} & \multicolumn{2}{c}{198M} & \multicolumn{2}{c|}{0.186} & 0.26s  & 7.73s  & 0.62s & 64GB  & 75GB & 65GB \\
        
        \midrule
         &  &  &  & &  & &    & &  &  &  && &&&   BF  & $k$-d & SNN & BF  & $k$-d & SNN \\
        \midrule
        \multirow{3}{*}{\shortstack[1]{Sem. \\ Robust.}} &  CIFAR-10-C  &  \multirow{2}{*}{$20k$}  & \multirow{3}{*}{$384^{*}$} &  \multirow{3}{*}{$L_{2}$} & \multirow{2}{*}{ $7.5$} & 
        \cite{croce2020robustbench} & \multicolumn{2}{c}{36M} &  \multicolumn{2}{c|}{\multirow{3}{*}{ {Fig~\ref{fig:heat}}} } 
        & \cite{diffenderfer2021winning} & \multicolumn{2}{c}{268M} &  \multicolumn{2}{c|}{\multirow{3}{*}{ {Fig~\ref{fig:heat}}}} & 2.36  & 83.75  & 12.75 & 388 & 519 & 573\\
        
        &  CIFAR-100-C  &   &  &   & & \cite{gowal2020uncovering} & \multicolumn{2}{c}{267M} & \multicolumn{2}{c|}{}
        & \cite{diffenderfer2021winning}  & \multicolumn{2}{c}{269M} & \multicolumn{2}{c|}{} & 4.53  & 59.26 & 37.80 & 396  & 519& 573 \\

        & ImageNet&  $10k$  & &   &  $12.5$ & \cite{krizhevsky2012imagenet}& \multicolumn{2}{c}{61M} & \multicolumn{2}{c|}{} & \cite{tian2022deeper} & \multicolumn{2}{c}{86M} &   \multicolumn{2}{c|}{}
        & 32.78  & 0.2s  & 60.25 & 376 & 373 & 401 \\

          \midrule
         & \multicolumn{3}{l|}{\scriptsize $(*)$ from DINOv2\cite{oquab2023dinov2} embedding}  &  &  & & {\scriptsize\cite{ruoss2020learning}}&  {\scriptsize \cite{khedr2023certifair}}  & {\scriptsize\cite{ruoss2020learning}}&   {\scriptsize\cite{khedr2023certifair}} & & {\scriptsize \cite{ruoss2020learning}}&  {\scriptsize \cite{khedr2023certifair}} & {\scriptsize\cite{ruoss2020learning}} &{ \scriptsize\cite{khedr2023certifair}}  &   BF  & $k$-d & BDD & BF  & $k$-d & BDD \\
         \midrule
            
        \multirow{3}{*}{\shortstack[1]{Ind. \\ Fair.}} &  German  &  $1k$  & $31$ &  \multirow{3}{*}{$L_{\infty}$} & \multirow{3}{*}{ $0.16$} &  
         \multirow{3}{*}{\shortstack[1]{ \cite{ruoss2020learning}  \\
          \cite{khedr2023certifair}}}   &  {\scriptsize 3.3k} &{\scriptsize 1.5k} & {\scriptsize 0.0} & {\scriptsize2.9}
        &  \multirow{3}{*}{\shortstack[1]{  \cite{ruoss2020learning}  \\
         \cite{khedr2023certifair}}} &  {\scriptsize 3.3k} &{\scriptsize1.5k }& {\scriptsize0.0} &{\scriptsize 0.0 }& 0.31 & 0.30  & 0.74  & 220 & 223  & 235 \\
        
        & Adult  &  $48.8k$  & $15$ &   &  & & {\scriptsize 5.1k} & {\scriptsize 1.1k} &  {\scriptsize0.1} &{\scriptsize 23.2}   & & {\scriptsize 5.1k} & {\scriptsize 1.1k} & {\scriptsize 0.0} &{\scriptsize 2.5}   & 0.69  & 1.21  & 5.38& 250  & 360 & 275 \\
        
             & COMPAS&  $6.2k$  & $18$ &   &  &  &{\scriptsize 1.5k} & {\scriptsize 0.7k} &{\scriptsize 1.8} &{\scriptsize 55.4}   &   & {\scriptsize 1.5k} &{\scriptsize 0.7k} & {\scriptsize0.2} & {\scriptsize 40.6 } & 0.31  & 0.49  & 4.50  & 223 & 235 & 236  \\

            \bottomrule
            
        
        

    \end{tabular}
    \caption{Experimental setup and performance summary for robustness monitoring applications. For each experimental setting we compare:
    the detected \ior violations for the \emph{base model} and the \emph{robust model}; the processing time per input and the total memory required by our monitor implemented with various FRNN algorithms.
    }
    \label{tab:app_tpr}
\end{table*}

\subsection{Computational Performances of Monitors}
\label{subsec:performance}
It is expected that the average computation time of monitors will grow with respect to the length of decision sequences and the number of dimensions in the data, due to the increase in FRNN search complexities.
We demonstrate these trends empirically.

\smallskip
\noindent\textbf{Length of decision sequences.}
We used the HIGGS data set~\cite{higgs_280} because of its large volume of  10.5 million entries.
For each entry we generated a synthetic output, and then used our different monitors to sequentially run over the 10.5 million decisions.
We repeated this experiment with 24 and 12 dimensions, and for different values of $\epsilon_Q\in \{0.01,0.025,0.05\}$.
In Figure~\ref{fig:performance}, we report the rolling average processing time per input (with 100k window size) with respect to increasing length of the decision sequence. 
We observe that BDD-based monitors are fastest for low dimensions and large values of $\epsilon_Q$, outperforming $k$-d trees with $L_{\infty}$-norm. 
This is surprising given the stellar performance of $k$-d trees for the $L_2$ norm. 
Both SNN and brute force perform reasonably well across all our experiments. 

Furthermore, the BDD-based algorithm shows a non-monotonic trend with respect to $\epsilon_Q$: As $\epsilon_Q$ increases, the number of partitions decreases, so the BDDs get smaller and more efficient. 
However, this introduces more false positives, requiring the lower-level brute-force routine to engage more frequently, causing a decrease in performance.
Intuitively, the BDD-based monitor performs well when the input data is sparse, so that the false positives are less frequent. This is expected for high-dimensional data, although higher dimension would increase the computational cost. 
This can be balanced by parallelizing with a just enough number of parallel workers, s.t. the data in each parallel FRNN remains sparse.

\smallskip
\noindent\textbf{Number of dimensions and parallel processing units.}
We augmented ImageNet data ~\cite{deng2009imagenet} with Gaussian noise and synthetic decisions, and ran our monitor on sequences of $10,000$ labeled samples with varying number of dimensions, obtained by sampling random pixels from the image.
We compare various FRNN monitoring algorithms, both without parallelization and measuring computational time starting from an initial history of length $100$k, and with parallelization and measuring computational time from the beginning (see Algorithm~\ref{alg:parallel}).
In Figure~\ref{fig:performance}, we can observe that BDD-based monitors are the slowest as the dimensions increase, and $k$-d trees and SNN are somewhat comparable. 
Moreover, the plots show that parallelization drastically increases the viability of monitoring in high dimensions. The number of threads should be chosen as a function of the dimension, as we can observe that there exists a sweet spot in the trade-off between the dimensionality reduction and parallelization overhead. This behavior is especially pronounced for BDD-based monitoring, as was explained earlier.

\newcommand{\legendline}[1]{\includegraphics[width=0.033\linewidth]{figures/lines/#1.png}}

\begin{figure*}
    \centering
    \begin{tabular}{ccc | c}
        \includegraphics[width=0.235\linewidth]{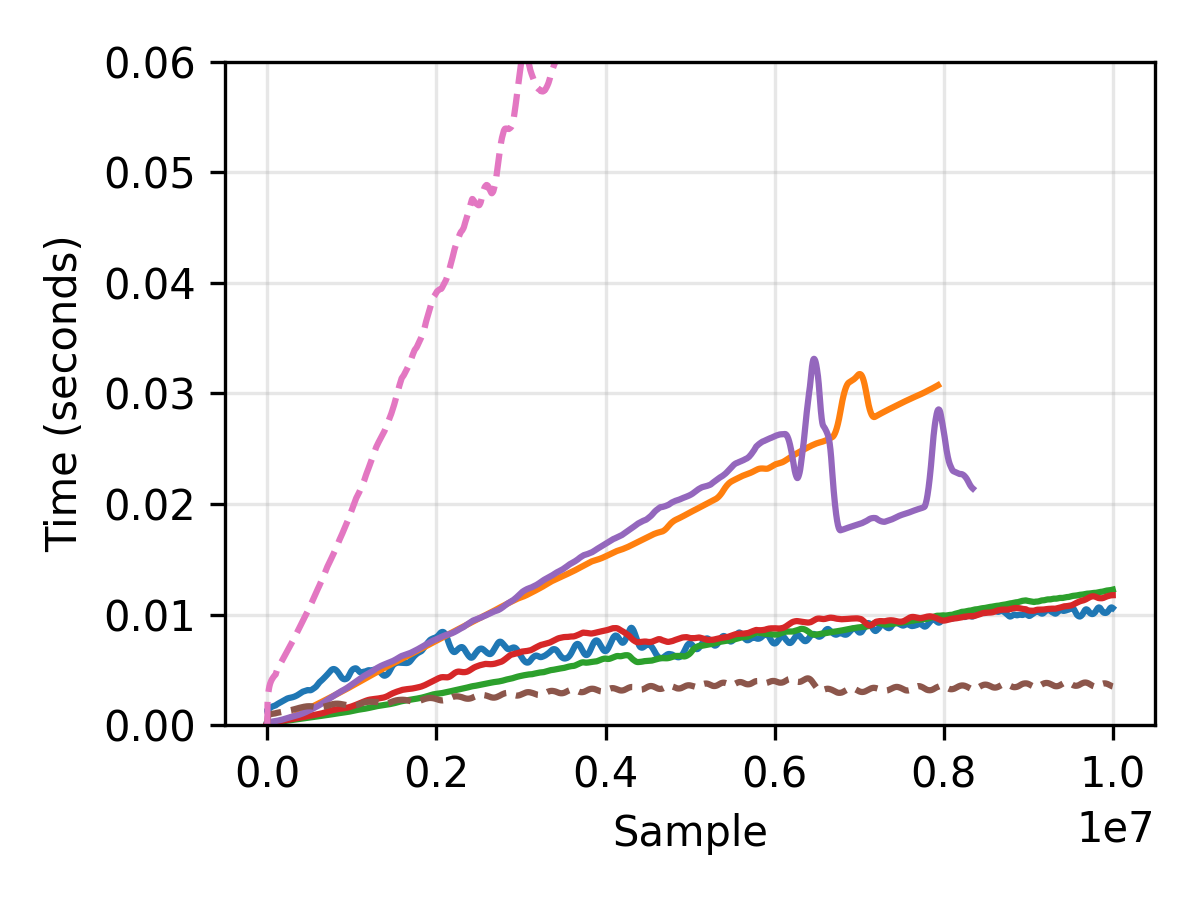} &
        \includegraphics[width=0.235\linewidth]{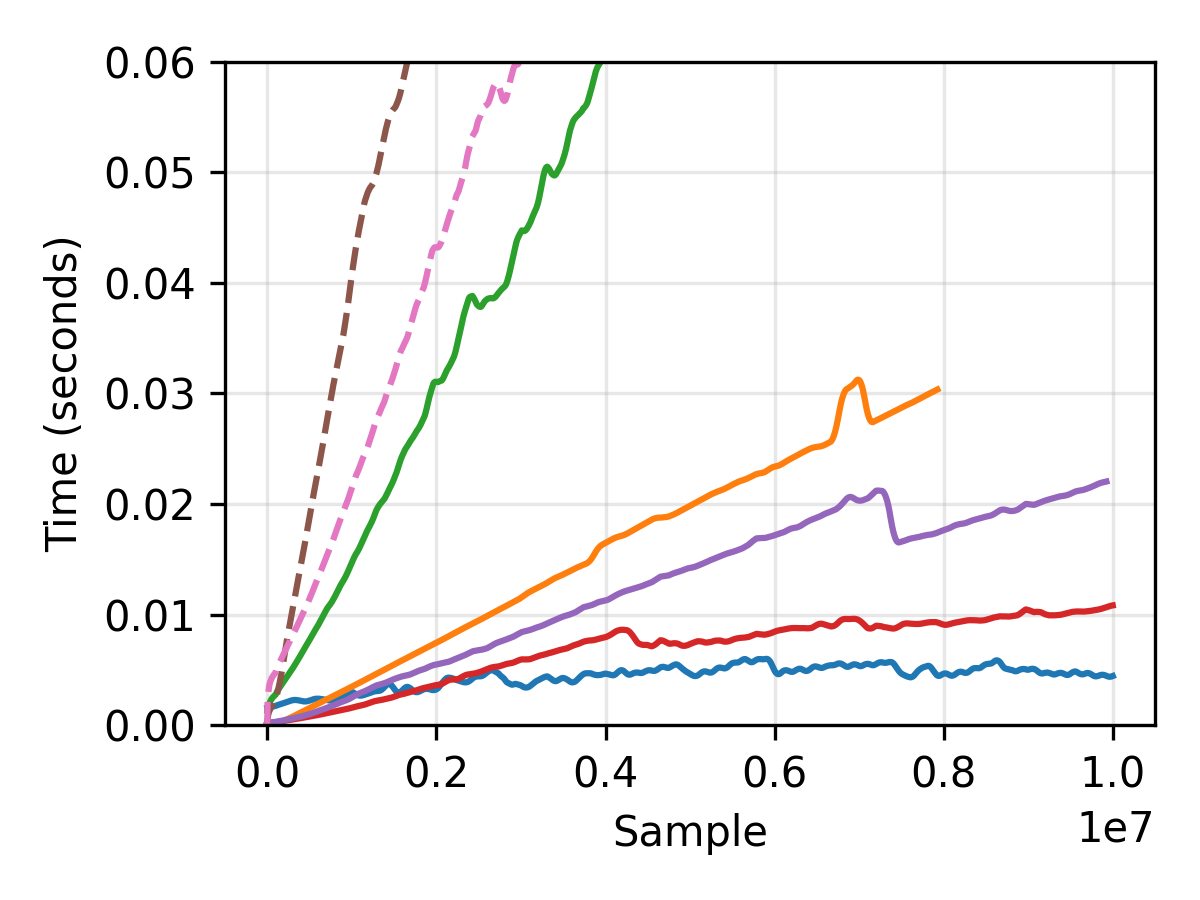} &
        \includegraphics[width=0.235\linewidth]{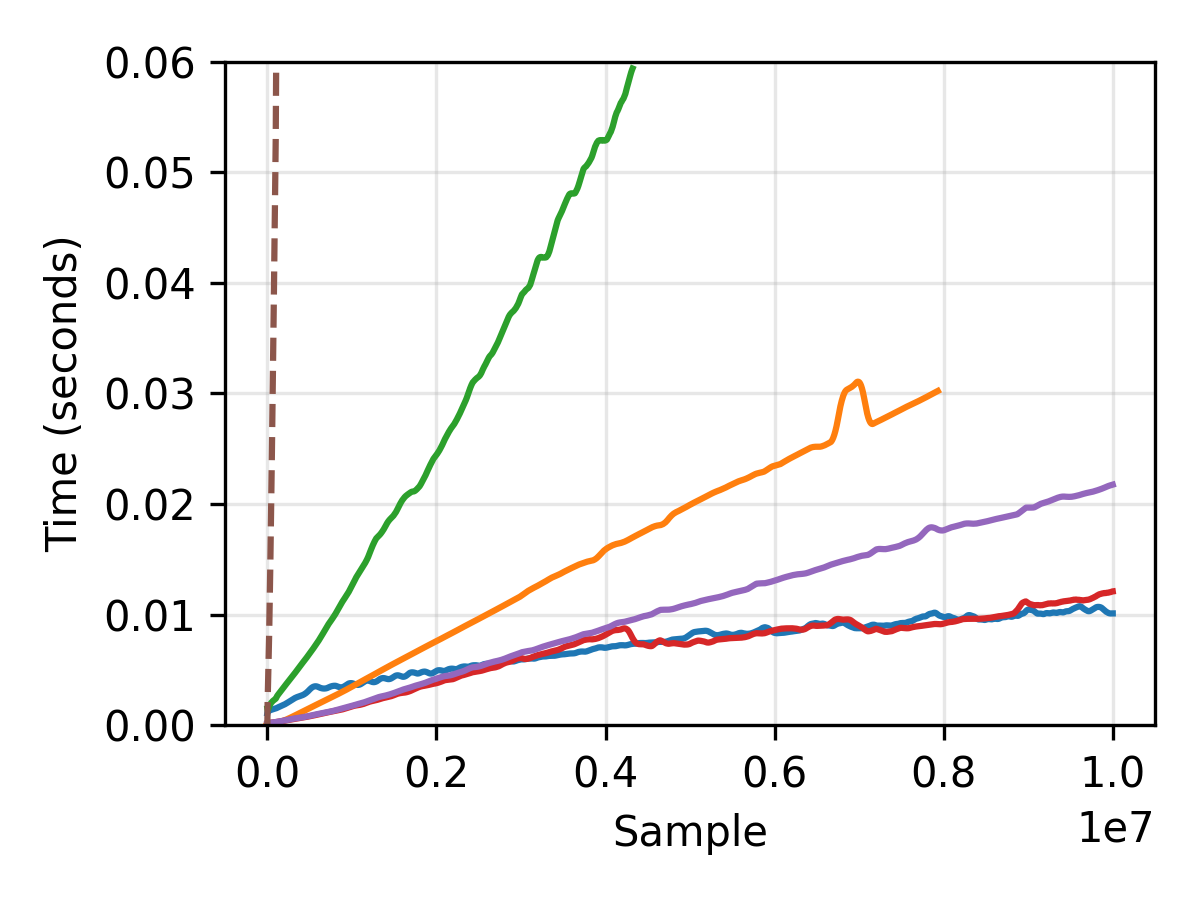} 
        & \includegraphics[width=0.235\linewidth]{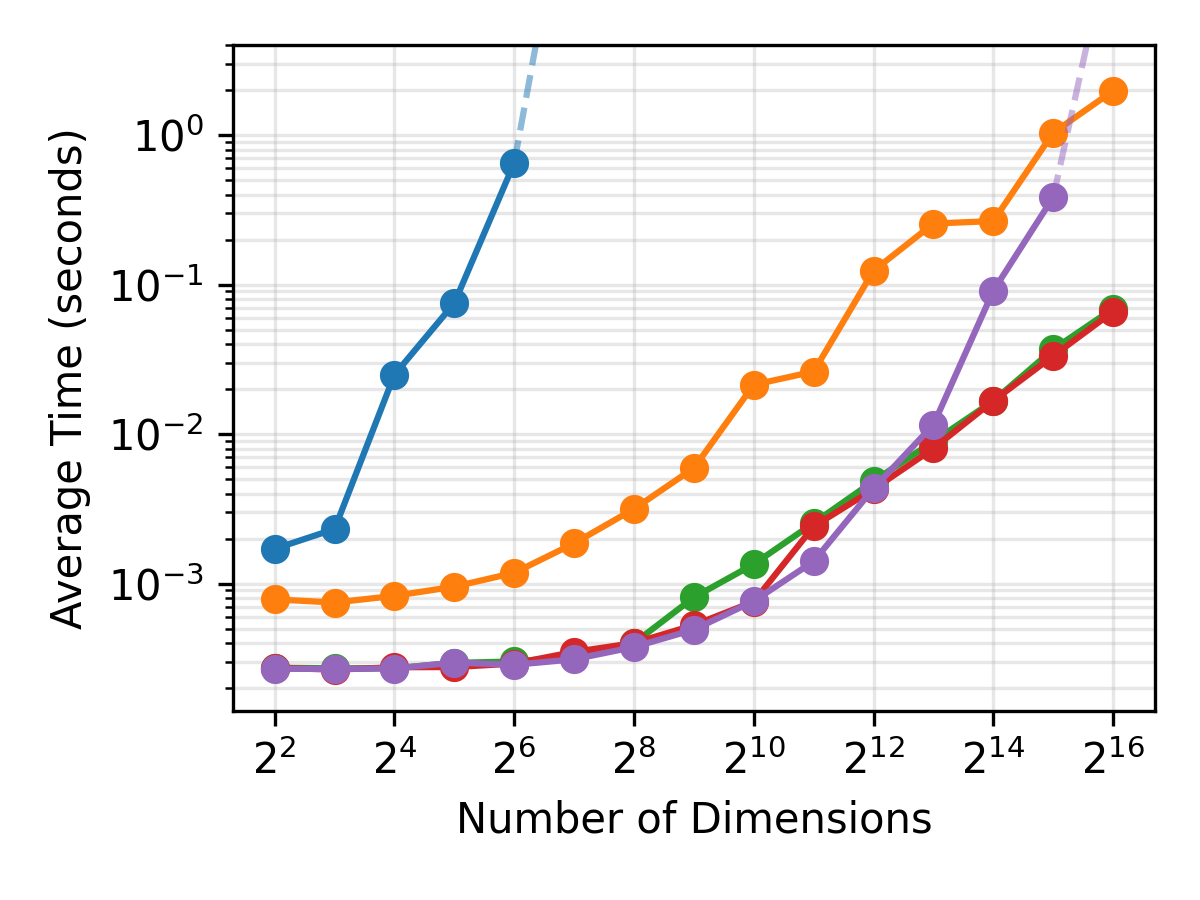}
        \\
        \includegraphics[width=0.235\linewidth]{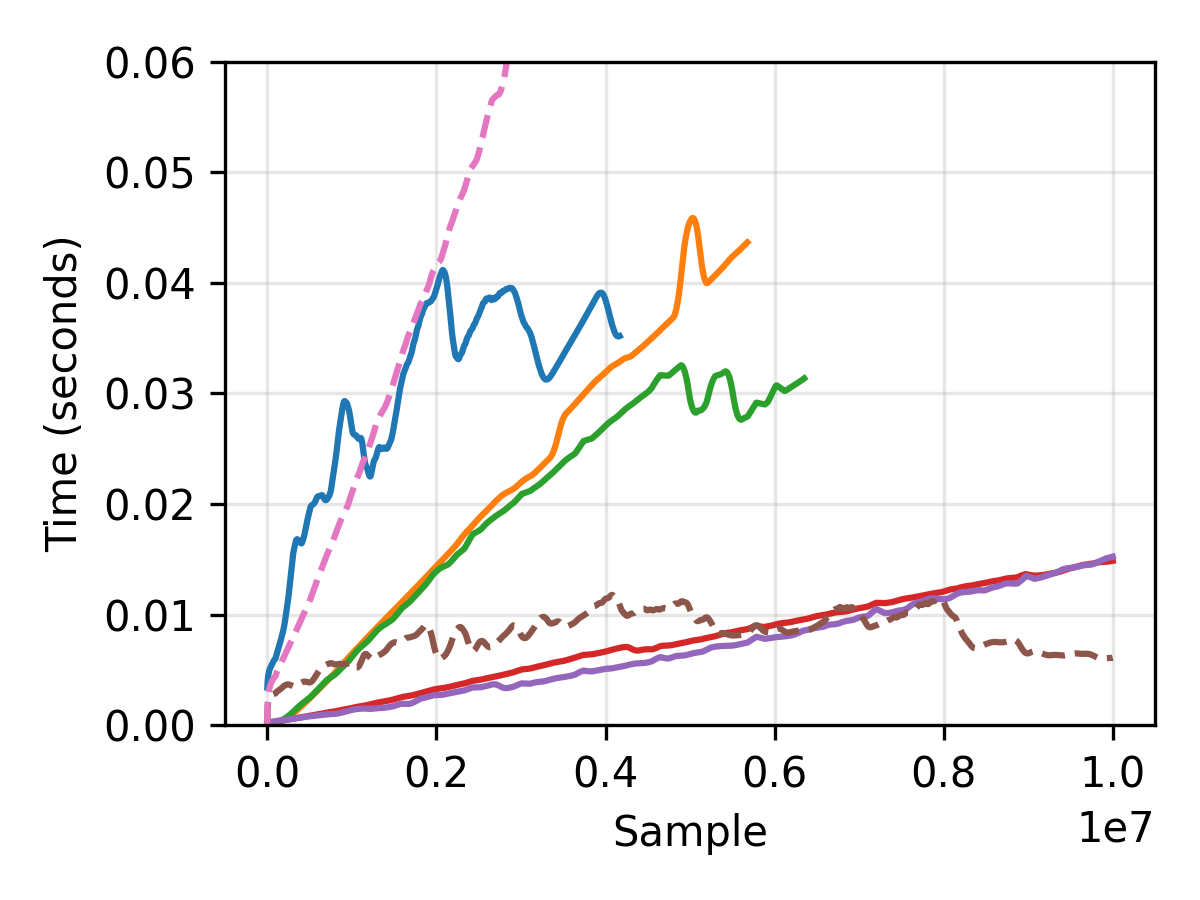} &
        \includegraphics[width=0.235\linewidth]{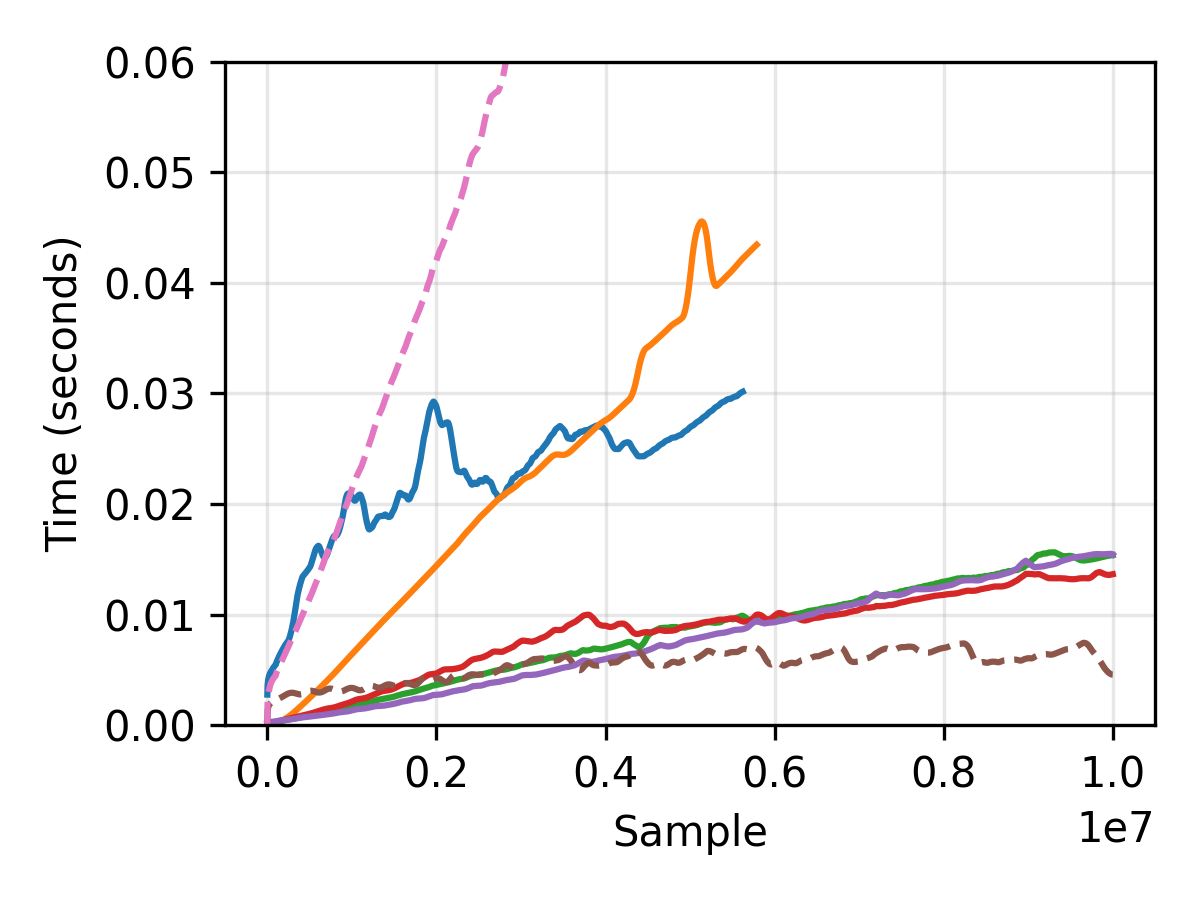} &
        \includegraphics[width=0.235\linewidth]{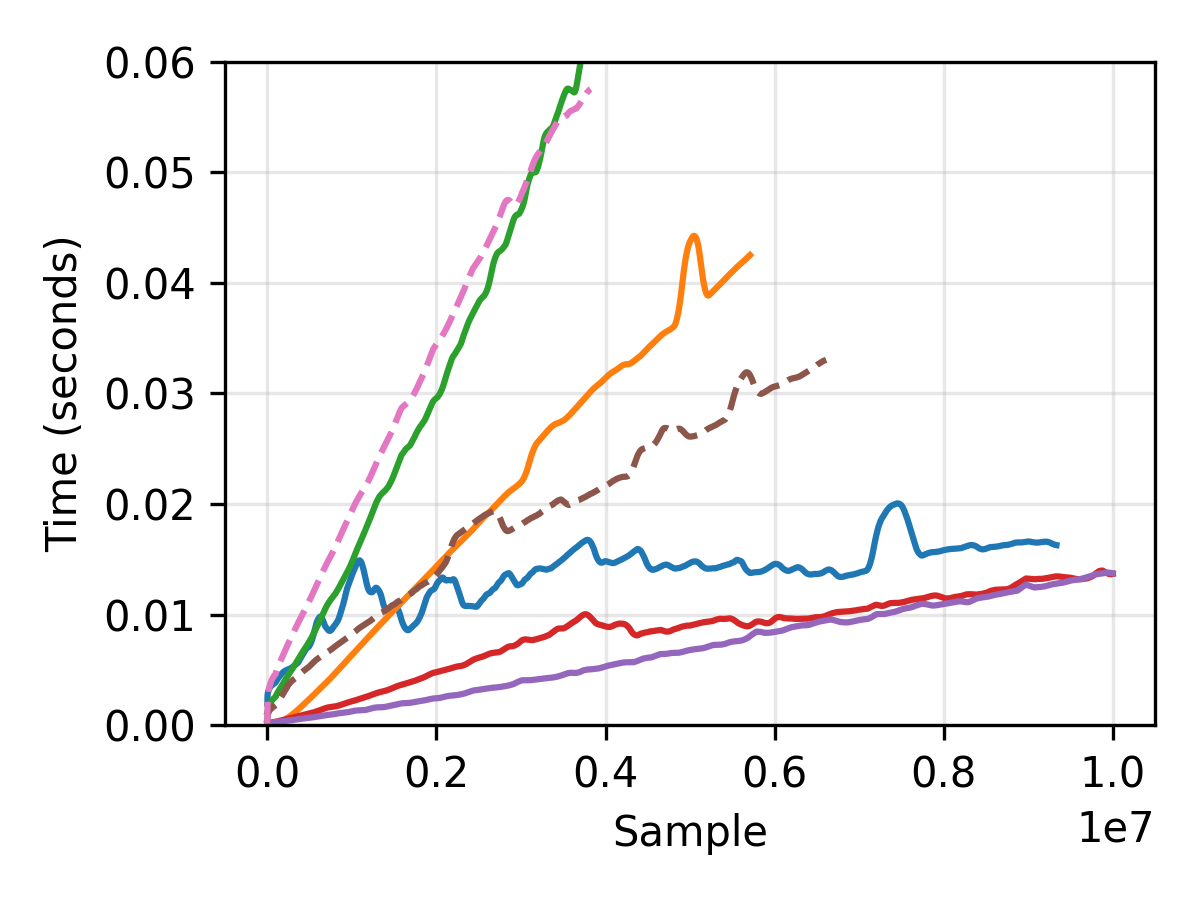} & 
        \includegraphics[width=0.235\linewidth]{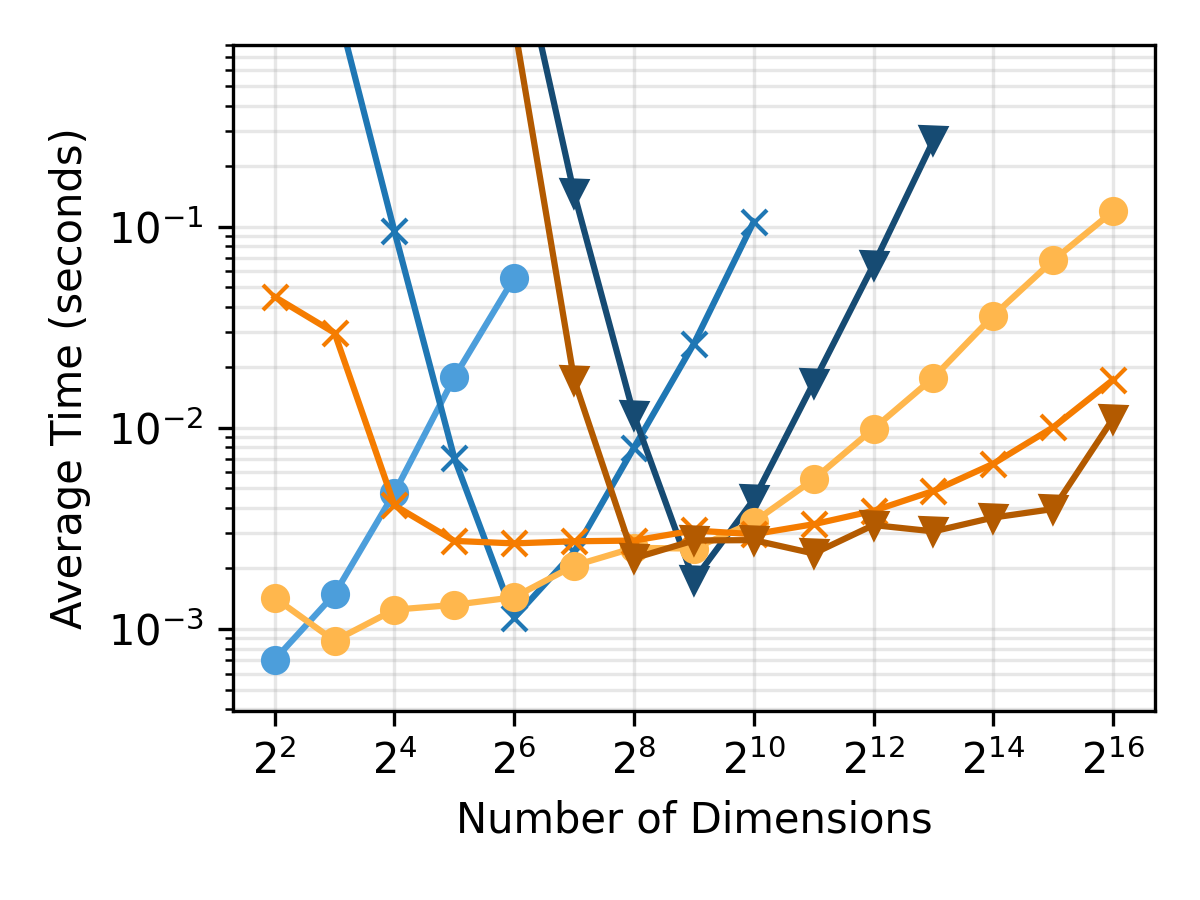}\\
    \end{tabular}
    \caption{LEFT: Performance comparison on the HIGGS dataset with 10 million entries. The rows correspond to 12 and 24 dimensional inputs respectively. The columns correspond to a $\epsilon$ of $0.01$,  $0.025$, and $0.05$ respectively.
    Legend:  BDD (\legendline{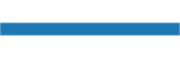}), Brute Force (\legendline{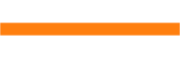}), Kd-tree $L_\infty$ (\legendline{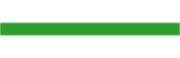}), Kd-tree $L_2$ (\legendline{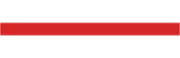}), SNN (\legendline{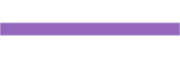}), 2-threaded BDD (\legendline{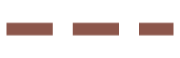}), 2-threaded Kd-tree (\legendline{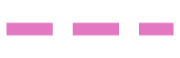}).
    RIGHT: The plot shows average processing time for 10k images from ImageNet: without parallelization after pre-loading 100k images (top); with parallelization after pre-loading 0 images (bottom). Parallelization plot only: BDD at 1 thread (\legendline{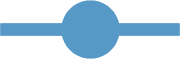}), BDD at 16 threads (\legendline{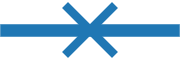}), BDD at 96 threads (\legendline{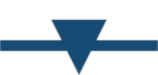}), Kd-tree at 1 thread (\legendline{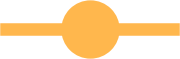}), Kd-tree at 16 threads (\legendline{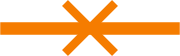}), Kd-tree at 96 threads (\legendline{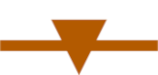})
}
    \label{fig:performance}
\end{figure*}

\section{Discusssions}

We propose \emph{runtime} \ior as a new variant of \ior properties that include adversarial robustness, semantic robustness, and individual fairness in one umbrella.
Runtime \ior requires the current run of a given AI decision maker be robust, and therefore is weaker than the traditional local or global \ior properties that require robustness to be satisfied even for inputs that may never appear in practice.
We propose \emph{monitors} for the detection of runtime \ior violations by deployed black-box AI models.
Our monitors build upon FRNN algorithms and use various optimizations, and their effectiveness and feasibility are demonstrated on real-world benchmarks.

Several future directions exist.
Firstly, we will incorporate more advanced FRNN algorithms in our monitors, like the ones with dynamic indexing and approximate solutions.
Secondly, our robustness (semantic and adversarial) case studies are only on image data sets, but there are other possibilities.
We plan to build monitors for spam filters that would warn the user if different verdicts were made for semantically similar texts from the past.
Finally, we will address various engineering questions about the monitoring aspect, like buffering new inputs while computation of previous inputs is still running, and distributed monitoring for networks of AI models.

\begin{acks}
This work was supported in part by the ERC project ERC-2020-AdG 101020093 and the SBI Foundation Hub for Data Science \& Analytics, IIT Bombay.
\end{acks}

\bibliographystyle{ACM-Reference-Format}
\bibliography{references}

\newpage
\appendix
\onecolumn

\begin{center}
    \Huge
    \textbf{Appendices}
\end{center}

\section{Pseudocode of FRNN Monitoring Algorithm using BDDs}
\label{app:FRNN using BDD}

\begin{algorithm}
    \caption{FRNN Monitoring Algorithm using BDDs}
    \label{alg:bdd monitor}
    \begin{algorithmic}[1]
        \Require Space $Q$, distance metric $d_Q$, constant $\epsilon_Q>0$
        \State BDD $f\gets \mathrm{bddZero}$ \Comment{initialize BDD for storing seen label vectors}
        \State BDD $f_d\gets \mathrm{bddZero}$ \Comment{initialize BDD for encoding the adjacency relation two label vectors}
        \State $\overline{Q}\gets \mathit{Discretize}(Q,d_Q,\epsilon_Q)$ \Comment{discretize the space $Q$ into boxes}
        \For{$\overline{q},\overline{q'}\in\overline{Q}$} \Comment{compute $f_d$ (one-time process)}
            \If{$\overline{q}$ and $\overline{q'}$ are adjacent}
                \State $f_d\gets f_d\lor f_{(\overline{q},\overline{q'})}$
            \EndIf
        \EndFor
        \While{$\mathrm{true}$}\Comment{monitoring begins}
            \State $q\gets \mathit{GetNewInput}()$ \Comment{$q$ is the next input point}
            \State BDD $g \gets f_d(\cdot,\overline{q})$ \Comment{$g$ represents the labels that are adjacent to $\overline{q}$ (the label of $q$)}
            \State BDD $h \gets f \land g$ \Comment{$h$ represents the labels that are in $g$ and also have appeared before}
            \If{$h=\mathrm{bddZero}$} 
                \State $R=\emptyset$ \Comment{no neighbor close to $q$ appeared in the past}
            \Else
                \If{$\left(h(b)=1 \Leftrightarrow b=\overline{q}\right)$}
                    \State $R\gets\Delta(\overline{q})$ \Comment{the past neighbors of $\overline{q}$ all had the label $\overline{q}$}
                \Else
                    \State $S \gets \mathit{getPoints}(h)$ \Comment{collect past neighboring labels (some can be false positives and are probably not coming from true neighbors of $q$)}
                    \State $R \gets \mathit{BruteForceFRNN}(\Delta(S),q;d_Q,\epsilon_Q)$ \Comment{low-level brute-force search to eliminate false positives from $S$}
                \EndIf
            \EndIf
            \State $f \gets f\lor f_{\overline{q}}$ \Comment{update the list of seen labels}
            \State $\Delta(\overline{q})\gets \Delta(\overline{q})\cup \{q\}$ \Comment{update the dictionary that maps seen labels to seen points}
            \State \textbf{output} $M(\rho,q)=R$ \Comment{the output of the monitor}
        \EndWhile
    \end{algorithmic}
\end{algorithm}

\section{Proof of Theorem~\ref{thm:runtime input-output robustness}}
\label{sec:proof}

\begin{theorem}
    Suppose $\epsilon_X, \delta_Z>0$ are constants and $X$ is infinite. 
    \begin{enumerate}
        \item Every decision-sequence of every globally $(\epsilon_X, \delta_Z)$-\ior classifier is runtime $(\epsilon_X, \delta_Z)$-\ior.
        \item If the decision-sequence of a classifier is runtime $(\epsilon_X, \delta_Z)$-\ior, the classifier is not necessarily globally $(\epsilon_X, \delta_Z)$-\ior.
    \end{enumerate}
\end{theorem}
\begin{proof}
    Claim (1): Let $D$ be an $(\epsilon_X, \delta_Z)$-\ior classifier and let $x_1, \dots, x_n \in X^n$ be a sequence of inputs. We know that \ior is satisfied for every two states in $X$. Hence, it must also be satisfied for $\{x_1, \dots, x_n\}\subseteq X$. 
    Claim (2): Let $\rho=(x_1, z_1), \dots, (x_n,z_n) \in X^n$ be a runtime $(\epsilon_X, \delta_Z)$-\ior decision sequence. Let $D$ be a classifier generating $\rho$. Let $x,x'\in X$ such that $d_x(x,x')\leq \epsilon_X \setminus \{x_1, \dots, x_n\}$. We define $D'$ identical to $D$, but for $x$ and $x'$ where we set $d_Z(D(x), D(x'))\geq \delta_Z$, i.e., $D'$ is not globally $(\epsilon_X, \delta_Z)$-\ior. 
\end{proof}

\newpage
\section{Counterexample pairs in semantic robustness monitoring}\label{sec:appendix:counter-examples}

\begin{figure}[htb]
    \centering
    \includegraphics[width=0.8\textwidth]{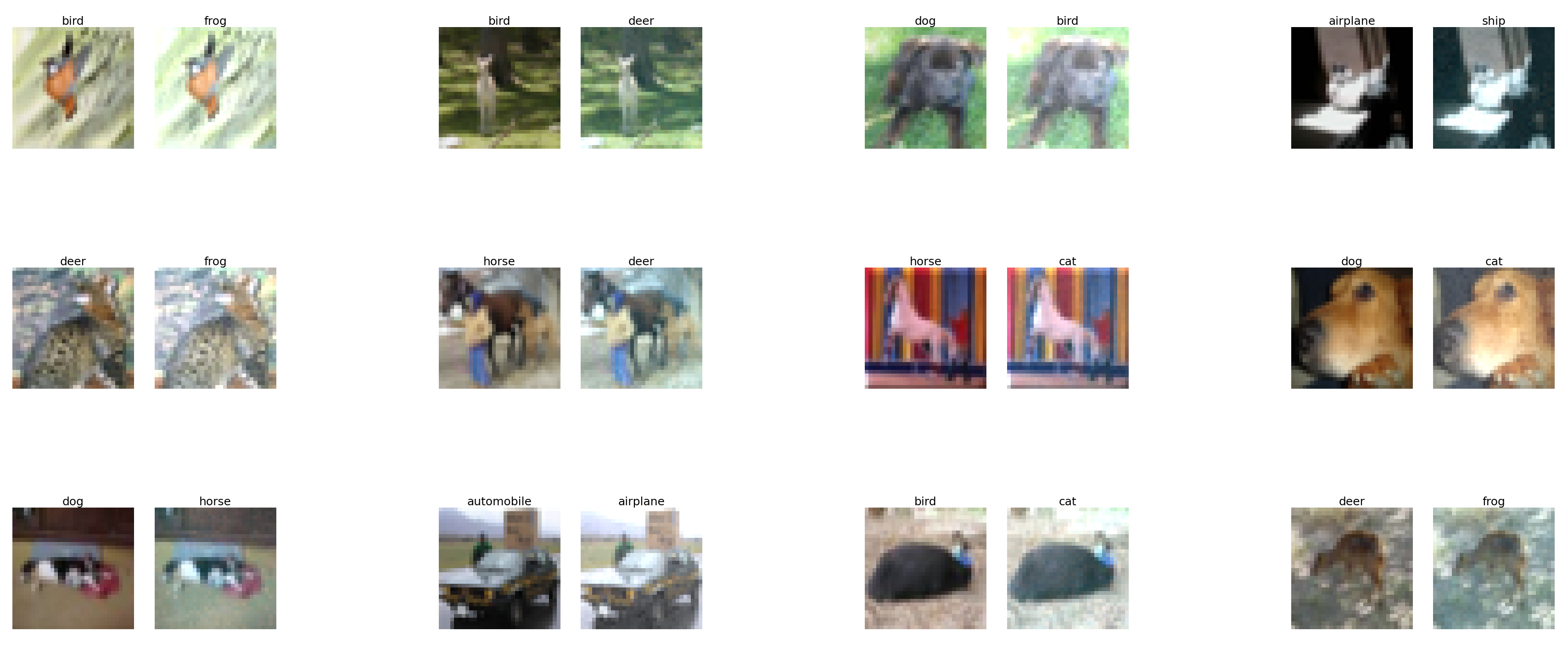}
    \caption{Sample of flagged inputs for semantic robustness on CIFAR-10C, Frost corruption, severity 2, baseline model \cite{croce2020robustbench}.}
    \label{fig:standard-frost}
\end{figure}

\begin{figure}[htb]
    \centering
    \includegraphics[width=0.8\textwidth]{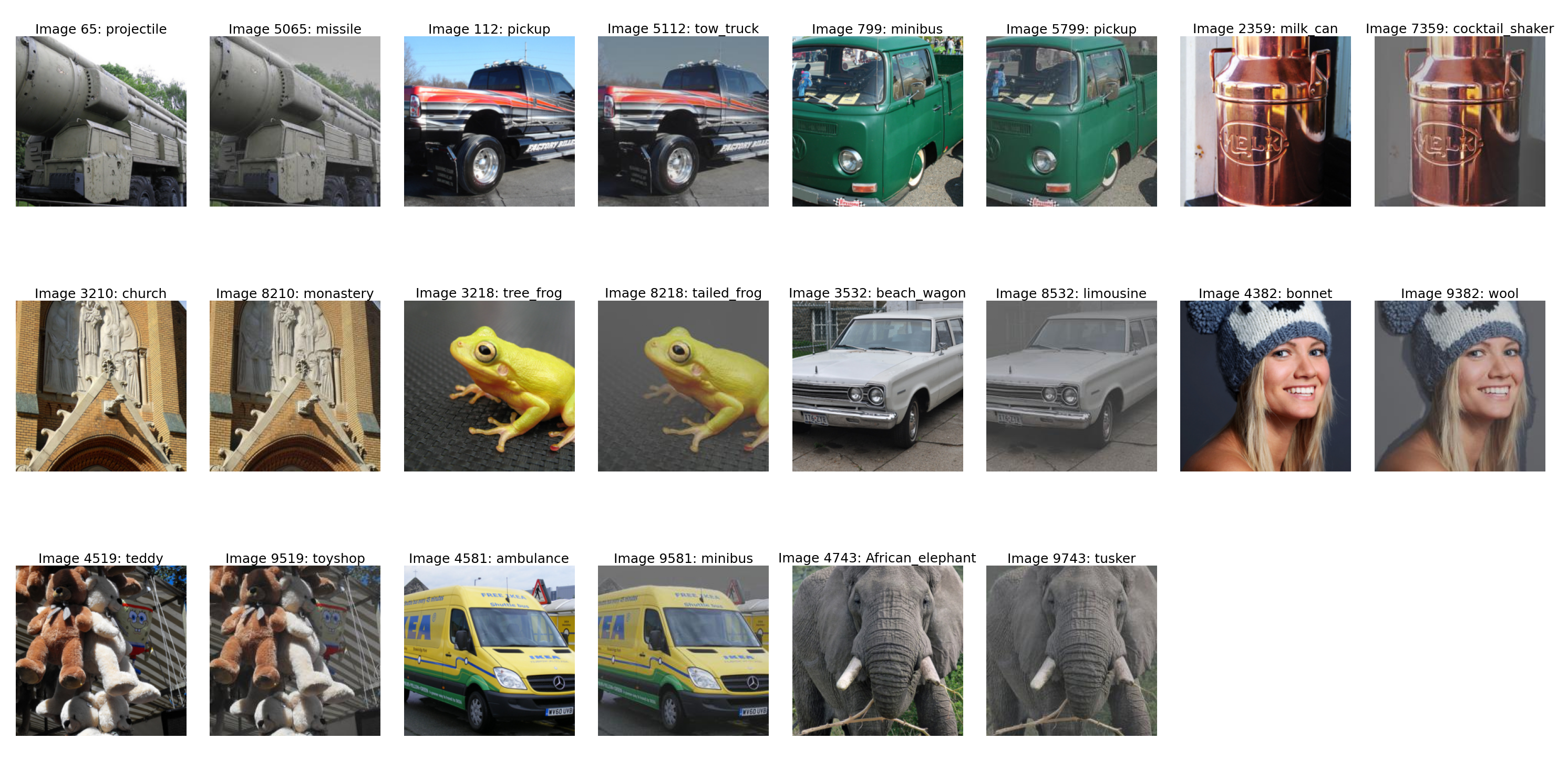}
    \caption{Flagged inputs for semantic robustness on Imagenet-3DCC, Fog corruption, severity 2, top model \cite{tian2022deeper}.}
    \label{fig:tian-deit-2}
\end{figure}

%


\end{document}